\documentclass[twoside]{article}

\usepackage[accepted]{aistats2026}
\usepackage[round]{natbib}

\usepackage{graphicx} 
\input{our_macros.sty}
\usepackage{dsfont}
\usepackage{hyperref}
\usetikzlibrary{patterns} 

\usepackage[capitalise]{cleveref}
\newcommand{\new}[1]{#1}
\newcommand{\Cov}{\mathrm{Cov}}
\newcommand{\err}{\mathrm{err}}
\newcommand{\polylog}{\mathrm{polylog}}
\newcommand{\leqlog}{\leq_{\log}}

\newcommand{\MMSE}{\mathrm{MMSE}}
\newcommand{\KL}{\mathrm{KL}}

\newtheorem{thm}{Theorem}[section]

\newtheorem{prop}[thm]{Proposition}
 
\newtheorem{cor}[thm]{Corollary}

\newcommand{\E}{\operatorname{\mathbb{E}}}
\renewcommand{\P}{\operatorname{\mathbb{P}}}

\newcommand{\pa}[1]{\left(#1\right)}
\newcommand{\ac}[1]{\left\{#1\right\}}
\newcommand{\cro}[1]{\left[#1\right]}

\renewcommand{\>}{\rangle}

\newcommand{\1}{\mathbf{1}}
\newcommand{\N}{\mathbb{N}}

\begin{document}

\twocolumn[

\aistatstitle{Statistical-computational Gap in Multiple Gaussian Graph Alignment}

\aistatsauthor{ Bertrand Even \And Luca Ganassali}

\aistatsaddress{Université Paris-Saclay \\ Laboratoire de mathématiques d’Orsay \\ Orsay, France \And  Université Paris-Saclay, CNRS, Inria \\ Laboratoire de mathématiques d’Orsay \\ Orsay, France}]

\begin{abstract}
    We investigate the existence of a statistical-computational gap in multiple Gaussian graph alignment. We first generalize a previously established informational threshold from \cite{vassaux2025} to regimes where the number of observed graphs $p$ may also grow with the number of nodes $n$: when $p \leq O(n/\log(n))$, we recover the results from \cite{vassaux2025}, and $p \geq \Omega(n/\log(n))$ corresponds to a regime where the problem is as difficult as aligning one single graph with some unknown "signal" graph. Moreover, when \new{$\log p = \omega(\log n)$}, the informational thresholds for partial and exact recovery no longer coincide, in contrast to the all-or-nothing phenomenon observed when \new{$\log p=O(\log n)$}. Then, we provide the first computational barrier in the low-degree framework for (multiple) Gaussian graph alignment.  We prove that when the correlation $\rho$ is less than $1$, up to logarithmic terms, low degree non-trivial estimation fails. Our results suggest that the task of aligning $p$ graphs in polynomial time is as hard as the problem of aligning two graphs in polynomial time, up to logarithmic factors. These results characterize the existence of a statistical-computational gap and provide another example in which polynomial-time algorithms cannot handle complex combinatorial bi-dimensional structures. 
\end{abstract}

\section{INTRODUCTION}

\paragraph{Graph alignment.} Graph alignment (also referred to as graph matching) is a fundamental statistical task which consists in finding an underlying correspondence between graphs, namely node relabelings which preserve most of the edges.
This problem has emerged in the 2000's across diverse application domains, including network privacy (\cite{Narayanan09}), computational biology (\cite{Singh08}), computer vision (\cite{CFVS04}), and natural language processing (\cite{haghighi05}). A common feature across these applications is that the
data naturally admits a graph-based representation: social networks, protein–protein interaction networks, three-dimensional images or meshes, and word or token embeddings.


Graph alignment has since been extensively investigated within the statistics and computer science communities, particularly in the two-graph setting, both for correlated Gaussian graphs \cite{GLM22spectral,ganassali22a,wu22settling,ding2022polynomialtimeiterativealgorithm} and for correlated Erdős–Rényi graphs \cite{fan20a,mao21a,GLM24,mao23}.

More recently, the problem of aligning multiple correlated graphs has attracted growing interest in the literature
(\cite{vassaux2025, ameen2025detectingcorrelationmultipleunlabeled}). 

\paragraph{Set-up.}
Multiple Gaussian graph alignment can be formalized as follows. We observe $p\geq 2$ undirected weighted graphs $G_1,\ldots, G_p$ on the set of edges $E=\binom{[n]}{2}$. We suppose that there exists $\pi^\star=\pa{\pi^\star_1, \ldots, \pi^\star_p}$ drawn uniformly on $\pa{\mathcal{S}_n}^p$, where $\mathcal{S}_n$ stands for the set of permutations of $[n]$, such that the following holds. Conditionally on $\pi^\star$, the edge weights $(G_{je})_{j \in [p], e \in E}$ have $\cN(0,1)$ marginal distribution, with covariance matrix defined for all $(j,e)\neq (j',e')$  by
\begin{equation}\label{eq:def_model}
    \Cov(G_{je}, G_{j'e'})=\rho \1_{\Pi^\star_j(e)=\Pi^\star_{j'}(e')},
\end{equation}
where we write, for $e=\ac{u,u'}$, $$\Pi^\star_j(e)=\ac{\pi^\star_j(u), \pi^\star_{j}(u')}\, ,$$ and where $\rho \in [0,1]$ is the correlation parameter, \new{that is the correlation between the Gaussian signal on aligned edges in \eqref{eq:def_model}. If $\rho=0$, the graphs are independent. If $\rho=1$, the graphs are identical, up to relabeling of the nodes.}

Our goal is to recover the $p$-tuple of permutations $\pi^\star=\pa{\pi^\star_1, \ldots, \pi^\star_p}$. Since the model is identifiable only up to a global permutation of $\pi^\star_1, \ldots,\pi^\star_p$, we are interested in controlling the proportion of misclassified points:
\begin{equation}\label{eq:errorperm}
    \err(\pi,\pi^{\star}) := \min_{\psi \in\mathcal{S}_{n}}\frac{1}{np}\sum_{u=1}^{n}\sum_{j=1}^{p}\1\ac{\psi(\pi_j(u))\neq \pi^{\star}_j(u)}\enspace .
\end{equation} 

Alternatively, the weighted graphs $G_1,\ldots, G_p$ can be sampled as follows. First draw $H_0$ a complete weighted graph with independent $\cN\pa{0,1}$ edges. Then, for all $j \in [p]$ and $e \in E$, set 
\begin{equation}\label{eq:def_model2}
    G_{je}=\sqrt{\rho} \, H_{0 \Pi^{\star}_j(e)}+\sqrt{1-\rho}\,  Z_{j e},
\end{equation}
\new{where $H_{0 \Pi^{\star}_j(e)}$ denotes the weight of edge $\Pi^{\star}_j(e)$ in the graph $H_0$}, and where $(Z_{je})_{e \in E}$, $j \in [p]$ are i.i.d. copies of $H_0$.

We say that an estimator $\widehat{\pi}:\mathbb{R}^{E\times p}\to \pa{\mathcal{S}_n}^p$ of $\pi^\star$ achieves:
\begin{itemize}
    \item \emph{perfect recovery} if $\err\pa{\widehat{\pi}, \pi^\star}=0$ with high probability\new{\footnote{\new{an event $\cE$ holds \emph{with high probability} if $\dP(\cE) \to 1$ when $n \to \infty$.}}} when $n \to \infty$, 
    \item \emph{partial recovery} if $\E\cro{\err(\widehat{\pi}, \pi^\star)} \to 0$ when $n \to \infty$,
    \item \emph{non-trivial  recovery} if $\dE[\err\pa{\widehat{\pi}, \pi^\star}]$ does not converge to $1$ when $n \to \infty$. 
\end{itemize}
Note that, since partial recovery does not require any specific rate of convergence of $\E\cro{\err(\widehat{\pi}, \pi^\star)}$, it is a weaker notion than exact recovery. Also note that if $\pi\in (\cS_n)^p$ is taken uniformly at random, then $\dE[\err\pa{\pi, \pi^\star}] \to 1$ when $n \to \infty$. This motivates the 'non-trivial recovery' terminology.

\paragraph{Related work.} 
The problem of aligning two Gaussian graphs has been studied both for characterizing the information-theoretic (IT) threshold of the problem (\cite{ganassali22a,wu22settling}) and for finding polynomial-time algorithms able to recover the single hidden permutation $(\pi^\star_1)^{-1}\circ \pi^*_2$ (\cite{GLM22spectral, fan20a,ding2022polynomialtimeiterativealgorithm}). 
For this problem, it is well known that the IT threshold for non-trivial recovery is exactly 
\begin{equation}\label{eq:IT_cas_p=2}
    \rho=\sqrt{\frac{4\log(n)}{n}}
\end{equation}and that, above this threshold, perfect recovery of $\pi^\star$ is feasible via maximum likelihood estimation.

However, this maximum likelihood estimator is not computable in polynomial time (see \cite{Pardalos94} and \cite{Makarychev14}). The state-of-the-art polynomial time procedure for aligning two correlated Gaussian graphs, developed by \cite{ding2022polynomialtimeiterativealgorithm} requires
a non-vanishing correlation $\rho$. This suggests the existence of a statistical-computational gap for the problem of aligning two correlated Gaussian graphs. 

More recently, alignment of multiple correlated Gaussian graphs has been studied by \cite{vassaux2025} and  \cite{ameen2025detectingcorrelationmultipleunlabeled}. In a regime where the number of graphs $p$ is a fixed constant and where $n$ tends to infinity, \new{
\cite{ameen2025detectingcorrelationmultipleunlabeled} studied the related detection problem, namely distinguishing the null hypothesis under which the $p$ graphs are $p$ independent Gaussian graphs, from the alternative where they are correlated according to the model~\eqref{eq:def_model}. The authors derived a sufficient condition
$$\rho \geq \sqrt{\frac{8\log(n)}{(n-1)p}}$$
for reliable detection, as well as a sufficient condition
$$\rho \leq \sqrt{\left(\frac{4}{p-1}-\eps\right)\frac{\log(n)}{n}}$$
for impossibility of detection.}

Then, the IT threshold for non-trivial recovery \new{is proven by \cite{vassaux2025}} to be exactly $$\rho=\sqrt{\frac{8\log(n)}{np}},$$ and here again, above this threshold, perfect recovery of $\pi^\star$ is possible with high probability. In this regime with fixed $p$ and $n\to \infty$, an \emph{all-or-nothing phenomenon} occurs: either non-trivial recovery is impossible, or perfect recovery is achievable.  

The case where the number of observed graphs $p$ can also grow with the number of nodes $n$ remains open, and raises the following question: 

\begin{itemize}
    \item[] \emph{(Q1). When the number of observed graphs 
$p$ grows with the number of nodes $n$, what is the information-theoretic threshold for perfect/exact recovery, and does an all-or-nothing phenomenon always occur?} 
\end{itemize}

Evidence for a statistical–computational gap for \ER graph alignment when $p=2$ has been established by recent results (\cite{ding2023lowdegreehardnessdetectioncorrelated, GMStree24, li25}). However, to the best of our knowledge, the existence of a statistical–computational gap for (multiple) Gaussian graph alignment has not been studied. We formalize this in the following question:  

\begin{itemize}
    \item[] \emph{(Q2). Can we provide evidence for the existence of a statistical–computational gap for multiple Gaussian graph alignment?} 
\end{itemize}

\paragraph{The low-degree model of computation.}
In many high dimensional statistical problems, such as 
sparse PCA \cite{HopkinsFOCS17}, planted clique \cite{Barak19}, clustering \cite{lesieur2016phase, even24, even2025a} or Gaussian graph alignment which is the focus of this paper, state-of-the-art polynomial time algorithms fail to reach the statistical performance provably achievable by algorithms free of computational constraints. These observations have led the community to conjecture the existence of \emph{statistical-computational gaps}, 
that is gaps between the best polynomial-time performance and the best performance with no computational constraints, for given problems. In light of this conjecture, a cornerstone of research in high-dimensional statistics is to determine rigorous lower bounds for specific classes of algorithms. 

In order to address 
average-case complexity, lower bounds are obtained for different models of computation, through the S-o-S hierarchy (\cite{HopkinsFOCS17}), the overlap gap property \cite{gam21}, statistical query (\cite{brennan20a}) and low-degree polynomials (\cite{WeinSchramm, hopkins2018statistical, wein2025insight}). 

In this paper, we will be interested in the low-degree model of computation: we consider the class of algorithms which can be written as multivariate polynomials in the observation with degree $D$. The low-degree model of computation has recently attracted considerable attention due to its ability to provide lower bounds 
matching state-of-the-art upper-bounds for polynomial time algorithms in different models: community detection (\cite{HopkinsFOCS17, sohnwein, carpentier2025lowdegreelowerboundsorthonormal,carpentier2025phasetransitionstochasticblock}), clustering (\cite{even24, even2025a}), and among others, \ER graph matching (\cite{ding2023lowdegreehardnessdetectioncorrelated,li25}). We refer to \cite{wein2025insight} for a recent survey. In this framework, the failure for a given task of polynomials of degree $D = O\pa{\log(T)^{1+\eta}}$, where $T$ is the dimension of the observation, is taken as an evidence for the failure of all polynomial-time algorithms. 

For estimation problems, the seminal work by \cite{WeinSchramm} offers a starting point for analyzing the performance of the best low-degree polynomial of degree $D$. Our analysis heavily builds upon their Theorem 2.2, which provides a formula for the case of Gaussian additive models.

\paragraph{Our contribution.} In this paper, we provide answers to questions (Q1) and (Q2) previously defined. 

For (Q1), we generalize the IT threshold $\rho=\sqrt{\frac{8\log(n)}{np}}$ \cite{vassaux2025} to regimes where $p$ is not necessarily constant. We characterize the informational thresholds for partial and perfect recovery, which are respectively $$\rho \gtrsim \frac{\log(n)}{n}\vee \sqrt{\frac{\log(n)}{np}}\enspace,$$
and $$\rho \gtrsim \frac{\log(np)}{n}\vee \sqrt{\frac{\log(np)}{np}}\enspace,$$
where $\gtrsim$ hides multiplicative constants. Moreover, between these two thresholds, we obtain an exponential decay of the error in terms of $n,p$ and $\rho$.

In particular, when $\log(p) = \omega(\log(n))$, these two thresholds differ significantly (by more than a multiplicative constant), and the all-or-nothing phenomenon established in \cite{vassaux2025} vanishes; there is now a different IT threshold for partial and exact recovery. 

For (Q2), we provide a computational threshold in the low-degree (LD) framework. We prove that low-degree polynomials fail at non-trivial recovery whenever
$$ \rho\leqlog 1,$$ where $\leqlog$ hides logarithmic factors. 

\paragraph{Main message.} These results suggest that, while observing a large number $p \geq 2$ of correlated Gaussian graphs helps information-theoretically in estimating $\pi^\star$, the computational difficulty of the problem remains equivalent up to logarithmic factors to the case $p=2$, that is to that of aligning two correlated Gaussian graphs.

Moreover, if the low degree conjecture is true, performing polynomial time alignment of $p$ Gaussian graphs requires the correlation $\rho$ to be at least of order $\frac{1}{\polylog(n,p)}$.  

Therefore, an optimal method up to some logarithmic factors is to perform pairwise alignment of the $p$ graphs using the algorithm of \cite{ding2022polynomialtimeiterativealgorithm}. In other words, there is no significant gain at a computational level in having access to multiple graphs.
We summarize our results in the phase diagram on \cref{fig:phase_diagram}. 


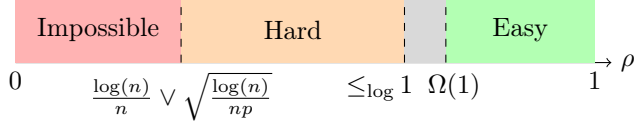
\begin{figure}
\centering
\begin{tikzpicture}[scale=1.1]

\draw[->] (0,0) -- (7.2,0) node[right] {$\rho$};
\draw (0,0) -- (0,0.8);

\node[below] at (0,0) {$0$};
\node[below] at (7,0) {$1$};

\fill[red!30]    (0,0) rectangle (2,0.8);   
\fill[orange!30] (2,0) rectangle (4.7,0.8); 
\fill[gray!30]   (4.7,0) rectangle (5.2,0.8); 
\fill[green!30]  (5.2,0) rectangle (7,0.8); 

\draw[dashed] (2,0) -- (2,0.8);
\draw[dashed] (4.7,0) -- (4.7,0.8);
\draw[dashed] (5.2,0) -- (5.2,0.8);

\node at (1,0.4) {Impossible};
\node at (3.35,0.4) {Hard};
\node at (6.1,0.4) {Easy};

\node[below] at (2,0) {$\tfrac{\log(n)}{n}\vee \sqrt{\tfrac{\log(n)}{np}}$};
\node[below] at (4.4,-0.05) {$\leqlog\,$1};
\node[below] at (5.3,-0.02) {$\Omega(1)$};

\end{tikzpicture}

\caption{\label{fig:phase_diagram}Statistical-computational landscape of partial recovery in multiple Gaussian graph alignment. The grey region is not fully covered by our results, since our lower bound involve logarithmic factors in $n$ and $p$.}
\end{figure}

\paragraph{Outline of the paper.}
We derive the maximum likelihood estimator and provide upper and lower information-theoretic bounds in Section \ref{sec:informational}, present the computational lower bound in Section \ref{sec:computational}, and discuss these results along with their connections to the literature in Section 4. All proofs are deferred to the appendices, while the main text provides an interpretation of the results. 

\paragraph{Notations.}
Throughout this paper, we adopt the following notations. 
We write $f(n,p) \lesssim g(n,p)$ if there exists a positive constant $c>0$ such that $f(n,p) \leq c g(n,p) $ \new{for all $n,p \geq n_0$ for some constant $n_0 \geq 1$}. We write $f(n,p) \asymp g(n,p)$ if both $f(n,p) \lesssim g(n,p)$ and $f(n,p) \gtrsim g(n,p)$. We write $f(n,p) \leqlog g(n,p)$ \new{if there exist polynomials $P,Q$ independent of everything else such that $f(n,p) \leq \frac{P(\log(n),\log(p))}{Q(\log(n),\log(p))} g(n,p) $} \new{for all $n,p \geq n_0$ for some constant $n_0 \geq 1$}. 
For $n \geq 1$, We write $[n]$ for the set $\set{1, \ldots, n}$, $\mathcal{S}_n$ for the set of permutations of $[n]$, and $N = \binom{n}{2}$ be the number of edges per graph. $E$ denotes the edge set $\binom{[n]}{2}$. 

We denote by $u,v,w,u',v',w' \in [n]$ the nodes of a graph, and by $j,k,\ell \in [p]$ the indices of graph copies. The undirected edges of a graph $G_j$ are denoted by $e,e',f,f' \in E$, and we write $G_{je}$ to refer to a particular edge $e$ of graph $G_j$. 

For each graph $G_j$, we consider node permutations $\pi_j \in \mathcal{S}_n$, and their corresponding edge permutations $\Pi_j \in \mathcal{S}_N$, defined by
$$
\Pi_j(e) = \{\pi_j(u),\pi_j(v)\} \in E \quad \text{for } e = \{u,v\} \in E.
$$
It is important to note that $\Pi_j$ and $\pi_j$ are distinct objects: the former acts on edges in $E$, while the latter acts on nodes in $[n]$. 

We denote by $\P$ 
the joint law of $(\pi^\star,G_1,\ldots, G_p)$ and, for $\pi\in (\cS_n)^p$, we denote by $\P_{\pi}$ the law of $(\pi^\star,G_1,\ldots, G_p)$ conditionally on $\pi^\star=\pi$. 
We write $\E$ (resp. $\E_\pi$) for expectation over $\P$ (resp. $\P_\pi$). 

\section{INFORMATION BARRIERS FOR EXACT AND PARTIAL RECOVERY}\label{sec:informational}

In this section, 
we provide the IT thresholds for partial and perfect recovery. In Section \ref{sec:MLE}, we derive the maximum likelihood estimator \new{and introduce useful objects for our analysis.}
\new{We analyze its performance in Section \ref{sec:MLE_small_p},} in regimes where $p$ is moderately small (namely $p\lesssim \rho^{-1}$). In Section \ref{sec:generalp}, we propose a different estimator for larger $p$ along with different upper bounds. In Section \ref{sec:inflowerbound}, we provide lower bounds which match the previously obtained upper bounds, for partial and perfect recovery. \new{The combination of those upper bounds and lower bounds fully characterize the information barriers for exact and partial recovery.}

\subsection{Derivation of the maximum likelihood estimator}\label{sec:MLE}

In line with model \eqref{eq:def_model}, define 
\begin{itemize}
    \item $G =(G_{je})_{j \in [p], e \in E}\in \R^{pN \times 1}$, viewed as a column vector.
    \item $H \in \R^{pN \times 1}$ the vector with $p$ copies of $H_0$, that is, for all $j \in [p]$, $e\in E$, ${H}_{je} = H_{e}$.
    \item ${\Pi}^{\star} \in \R^{pN \times pN}$ the aggregated block matrix of the edge permutations, that is, for all $j,j' \in [p]$, $e,e' \in E$, ${\Pi}^{\star}_{je,j'e'} = \1\ac{j=j'}\1\ac{\Pi^{\star}_j(e)= e'}$. 
    \item $G^\star \in \R^{pN \times 1}$ defined as $G^\star = \Pi^\star H$. For all $j \in [p]$, $e\in E$, $G^\star_{je} = H_{ \Pi^\star_{j}(e)}$.
    \item $Z \in \R^{pN \times 1}$ the vector of aggregated noise variables, $Z = (Z_{je})_{j \in [p], e \in E} \sim \cN(0,I_{pN})$.
\end{itemize}
A compact formulation of model \eqref{eq:def_model} is given by
\begin{equation}\label{eq:def_model_matrices}
    G = \sqrt{\rho} \, \Pi^{\star} H + \sqrt{1-\rho} \, Z = \sqrt{\rho} \, G^{\star} + \sqrt{1-\rho} \, Z \, .
\end{equation} Under $\P_{\pi}$, $G$ is a centered Gaussian vector with covariance matrix $\Sigma(\pi) \in \R^{pN \times pN}$ given by
\begin{align*}
    \Sigma(\pi) & = \dE[GG^T] = \rho \Pi \dE[HH^T] (\Pi)^T  + (1-\rho)\dE[ZZ^T]  \\
    & = \Pi [ (\rho J_p + (1-\rho)I_p) \otimes I_N] (\Pi)^T,
\end{align*} where $\otimes$ denotes the Kronecker product, using that $\Pi\Pi^T=I_{pN}$. Unsurprisingly, $\Sigma(\pi)$ is invariant under global permutation, just as the law of $G$. 
Note that for $0 \leq \rho <1$, $\Sigma(\pi)$ is invertible with inverse
{\small
\begin{align}\label{eq:inverse_cov}
    & \Sigma^{-1}(\pi)  = \Pi [ (\rho J_p + (1-\rho)I_p)^{-1} \otimes I_N] \Pi^T \nonumber \\ 
    & = \Pi \left[ \left( - \frac{\rho}{(1-\rho)(1+(p-1)\rho)} J_p + \frac{1}{1-\rho}I_p \right) \otimes I_N \right] \Pi^T \nonumber \\ 
    & = - \frac{\rho}{(1-\rho)(1+(p-1)\rho)} \Pi (J_p \otimes I_N) \Pi^T + \frac{1}{1-\rho} I_{pN} \, .
\end{align}}
\normalsize

Given the shape of the Gaussian likelihood and since $\det \Sigma(\pi)$ does not depend on $\pi$, the maximum likelihood estimator (MLE) for this problem is given by 
\begin{align}
    \label{eq:MLE}
    \widehat{\pi} & \in \argmin_{\pi \in (\cS_n)^p} \, G^T \Sigma^{-1}(\pi) G \nonumber \\ 
    & = \argmax_{\pi \in (\cS_n)^p} \, G^T\Pi (J_p \otimes I_N) \Pi^T G \, .
\end{align}

For any $\pi \in (\cS_n)^p$, define its \emph{normalized edge alignment matrix} $B(\pi) \in \R^{pN \times pN}$  as 
\begin{equation}\label{eq:def_B}
     [B(\pi)]_{je,j'e'} = \frac{1}{p}\1\ac{\Pi_j(e)=\Pi_{j'}(e')} \, .
\end{equation}
Define $\cB := B((\cS_n)^p)$. Note that $\pi \mapsto B(\pi)$ is a bijection from $(\cS_n)^p / \equiv$ onto $\cB$, where $\pi \equiv \pi'$ in $(\cS_n)^p$ if and only there exists $\phi \in \cS_n$ such that $\pi'_j = \phi \circ \pi_j$ for all $j \in [p]$.  

Note also that any $B \in \cB$ satisfies
\begin{equation*}
    B^T=B, B^2=B, B \1 = \1, \Tr(B)=N\, .
\end{equation*}
and that in \eqref{eq:inverse_cov}, 
\begin{equation}\label{eq:from_B_to_cov}
    \Pi^{\star} (J_p \otimes I_N) (\Pi^{\star})^T = p B^\star, 
\end{equation} where $B^{\star} :=  B(\pi^{\star})$. Finding the MLE in \eqref{eq:MLE} is thus equivalent to solving
\begin{equation}\label{eq:MLE_on_B}
    \widehat{B} = \new{\argmax_{B \in \cB}} \langle GG^T, B \rangle,
\end{equation} and we find back the MLE $\widehat{\pi}$ taking any $\widehat{\pi} \in B^{-1}(\widehat{B})$. 

\subsection{\new{Moderately small $p$: exact recovery via the MLE}}\label{sec:MLE_small_p}
In order to prove an IT upper bound, we show that the maximum likelihood estimator in its matrix version \eqref{eq:MLE_on_B} enables to recover $B^\star$, and thereby $\pi^{\star}$, with the desired error guarantees. 
\new{We are able to analyze directly the MLE in the case where $p \lesssim {\rho}^{-1}$, and will address the case $ p \gtrsim  {\rho}^{-1}$ slightly differently in the next section.}
  


\begin{thm}\label{thm:upperboundIT}
There exists numerical constants $C,c_1,c_2$ and $n_0$ such that the following holds. Whenever $\rho\geq c_1\sqrt{\frac{\log(n)}{np}}$, $p\leq C \rho^{-1}$ and $n\geq n_0$, we have

$$\P\cro{\err(\widehat{\pi}, \pi^{\star})\neq 0}\leq \exp\pa{-c_2np\rho^2},$$ where $\widehat{\pi}$ is the MLE defined in \eqref{eq:MLE}.
\end{thm} 

\new{The above Theorem is proved in Appendix \ref{prf:upperboundIT}. We remark that the probability in the right-hand side of the inequality of \cref{thm:upperboundIT} is upper-bounded by $n^{-c_1c_2}=o(1)$, hence the feasibility of exact (and partial) recovery under these assumptions.}
In this regime, we recover the threshold from \cite{vassaux2025}, up to some multiplicative numerical constant, which we generalize to higher values of $p$. 

\new{Note also that under the assumptions of \cref{thm:upperboundIT}, $p \lesssim \rho^{-1} \lesssim \sqrt{\frac{np}{\log(n)}}$, which gives $p \lesssim \frac{n}{\log(n)}$. We will refer to this as the \emph{moderately small-$p$} regime.}




\subsection{\new{For larger $p$: partial recovery upper bound}}\label{sec:generalp}
\new{When $ p \gtrsim  {\rho}^{-1}$, the approach relies on computing the maximum likelihood estimator for the first $p'$ graphs, with  $p' \asymp {\rho}^{-1}$, and aligning the rest of the graphs pairwise with the aggregated signal on the first $p'$ graphs.}

Namely, we build an estimator $\hat{\pi} \in (\cS_n)^p$ with the following procedure.

\begin{enumerate}
    \item Let $p'=C\rho^{-1}$, where $C$ is determined by Theorem \ref{thm:upperboundIT}, set $(\hat{\pi}_1, \ldots, \hat{\pi}_{p'})$ to be the MLE of $(G_1, \ldots,G_{p'})$ as defined in \eqref{eq:MLE}. Let $\widehat G'$ be the graph defined by $$\widehat G'_{e}=\frac{1}{p'}\sum_{j\leq p'}G_{j \pa{\widehat\Pi_j}^{-1}(e)}\enspace ;$$
    \item For all $j>p'$, compute the remaining $\widehat{\pi}_j$ aligning $G_j$ with $\widehat G'$ with the MLE estimator \eqref{eq:MLE} on $(\widehat G',G_j)$.
    \item Return the estimator 
    \begin{equation}\label{eq:tilde_pi}
        \widehat \pi = (\widehat{\pi}_1, \ldots, \widehat{\pi}_{p}) \, .
    \end{equation}
\end{enumerate}

Theorem \ref{thm:upperboundIT} suggests that, for $\rho$ large enough, $(\hat{\pi}_1, \ldots, \hat{\pi}_{p'})$ will recover exactly $(\pi^\star_1,\ldots, \pi^\star_{p'})$ with high probability and thus that $\widehat G'$ will be isomorphic to the graph $G'$ defined by  $$G'_{e}=\frac{1}{p'}\sum_{j\leq p'}G_{j \pa{\Pi^\star_j}^{-1}(e)}\enspace.$$ Then, appealing again to Theorem \ref{thm:upperboundIT}, one can control the probability of recovering $\pi^\star_j$, for $j>p'$, by aligning $G_j$ with $\widetilde G'$. We deduce the next theorem, proved in \cref{prf:upperboundITgeneralp}.

\begin{thm}\label{thm:upperboundITgeneralp}
    Let $C$ be the constant defined in \cref{thm:upperboundIT}. There exists numerical constants $c_1, c_2$ and $n_0$ such that the following holds. Whenever $\rho\geq c_1\pa{\frac{\log(n)}{n}\vee \sqrt{\frac{\log(n)}{np}}}$, \new{$p > C \rho^{-1}$} and $n\geq n_0$, we have $$\E\cro{\err(\hat{\pi}, \pi^\star)}\leq \exp\pa{-c_2 n\rho}\enspace,$$ 
    where $\widehat{\pi}$ is defined in \eqref{eq:tilde_pi}.
\end{thm}

\new{The exponential term in the right-hand side of the inequality of \cref{thm:upperboundITgeneralp} is upper-bounded by $n^{-c_1c_2}=o(1)$, hence the feasibility of partial recovery under these assumptions.}

Combining \cref{thm:upperboundIT} and \cref{thm:upperboundITgeneralp}, we deduce that partial recovery is possible with high probability whenever $$\rho\gtrsim \pa{\frac{\log(n)}{n}\vee \sqrt{\frac{\log(n)}{np}}}\enspace.$$

When the correlation is even higher, namely when 
\begin{equation}
    \rho \gtrsim \frac{\log(np)}{n}\vee \sqrt{\frac{\log(np)}{np}}
\end{equation} by a large enough constant, observe that 
\begin{itemize}
    \item \new{the probability of error $\P\cro{\err(\widehat{\pi}, \pi^{\star})\neq 0}$ in \cref{thm:upperboundIT} drops below $1/(np)^2$},
    \item \new{the expectation $\E\cro{\err(\hat{\pi}, \pi^\star)}$ in \cref{thm:upperboundITgeneralp} drops below $1/(np)^3$, and by Markov's inequality, 
    \begin{align*}
        \P\cro{\err(\widehat{\pi}, \pi^{\star})\neq 0} & = \P\cro{\err(\widehat{\pi}, \pi^{\star})\geq 1/(np)} \\ & \leq 1/(np)^2 \, .
    \end{align*}}
\end{itemize} 
\new{Since the MLE is always Bayes-optimal for perfect recovery, we deduce from the above that the MLE achieves perfect recovery with high probability when $\rho$ is large enough\footnote{\new{Theorems \ref{thm:upperboundIT} and \ref{thm:upperboundITgeneralp} show that only $\rho \gtrsim \frac{\log(np)}{n}\vee \sqrt{\frac{\log(n)}{np}}$ is needed for exact recovery. Our statement is equivalent since $\frac{\log(np)}{n}\vee \sqrt{\frac{\log(n)}{np}} \asymp \frac{\log(np)}{n}\vee \sqrt{\frac{\log(np)}{np}}$}.}.} 

\begin{cor}\label{cor:MLEperfect}
    There exists numerical constants $c_1, c_2$ and $n_0$ such that the following holds. Whenever $\rho\geq c_1\pa{\frac{\log(np)}{n}\vee \sqrt{\frac{\log(np)}{np}}}$ and $n\geq n_0$, we have, with probability at least $1-1/(np)^2$, $$\err(\widehat{\pi}, \pi^\star)=0 \, ,$$ where $\widehat{\pi}$ is the MLE defined in \eqref{eq:MLE}.
\end{cor}

\begin{remark}[On the exponential decay] 
    Theorems \ref{thm:upperboundIT} and \ref{thm:upperboundITgeneralp} show in particular that the error \new{$\E\cro{\err(\widehat{\pi}, \pi^\star)}$} exhibits an exponential decay with respect to $n\pa{\rho\wedge p\rho^2}$.
Such exponential behavior is reminiscent of what is observed in several other high-dimensional inference 
problems, including clustering (\cite{ndaoud2022sharp, giraud2019partial, even24}), multiple feature 
matching (\cite{even2025b}), and supervised classification (\cite{giraud2019partial}).
\end{remark}

\subsection{Informational lower bounds}\label{sec:inflowerbound}

The next result, proved in Appendix \ref{prf:lowerboundinf}, provides an information-theoretic lower bound for partial recovery.

\begin{thm}\label{thm:lowerboundinf}
    There exists numerical constants $n_0,c,c' >0$ such that the following holds. If $n\geq n_0$ and $\rho\leq c\pa{\frac{\log(n)}{n}\vee \sqrt{\frac{\log(n)}{np}}}$, then $$\inf_{\widehat{\pi}}\dE[\err(\widehat{\pi}(G), \pi^{\star})]\geq c',$$ where the infimum is taken over all measurable functions $\widehat{\pi}:\mathbb{R}^{E\times p}\to \pa{\mathcal{S}_n}^p$.
    Thus, partial recovery is impossible.
\end{thm}

Theorems \ref{thm:upperboundIT}, \ref{thm:upperboundITgeneralp} and \ref{thm:lowerboundinf} show that the estimator defined in \eqref{eq:tilde_pi} is optimal up to some numerical constant and that the IT threshold for partial recovery is 
\begin{equation}\label{eq:IT_partial}
    \rho\gtrsim \frac{\log(n)}{n}\vee \sqrt{\frac{\log(n)}{np}}\enspace.
\end{equation}
Let us now give some intuition for why the IT lower bound \eqref{eq:IT_partial} takes this particular form. Consider a genie-aided version of the problem in which we are given $\pi_2^\star = \cdots = \pi_p^\star = \mathrm{id}$ and we seek to recover $\pi_1^\star$. A sufficient statistic for estimating $\pi_1^\star$ is $(G_1, \overline{G})$, where 

$$\overline{G}_e = \frac{1}{p-1} \sum_{j=2}^p G_{je} \, .$$

In these two graphs, the edge-weight variances are $1$ and $\tfrac{1+(p-2)\rho}{p-1}$, respectively, while each correlated edge pair has covariance $\rho$. Thus, the problem reduces to aligning the two correlated Gaussian graphs $G_1$ and $\overline{G}$, with effective edge correlation 
\begin{equation}
    \label{eq:def_rhoprime}
    \rho' = \frac{\sqrt{p-1}\,\rho}{\sqrt{1 + (p-2)\rho}}
    \;\lesssim\; \sqrt{\rho} \,\vee\, \sqrt{p}\,\rho \, .
\end{equation}
For $p=2$, the previously known information-theoretic threshold \eqref{eq:IT_cas_p=2} implies that a necessary condition for partial recovery of $\pi_1^\star$ is 
$$
\rho' \gtrsim \sqrt{\tfrac{\log n}{n}},
$$
and combining this with \eqref{eq:def_rhoprime} yields exactly \eqref{eq:IT_partial}.

In the regime where $p \lesssim n/\log(n)$, the threshold for partial recovery reduces to $\rho \gtrsim \sqrt{\frac{\log(n)}{np}}$, which matches, up to constants, the condition obtained by~\cite{vassaux2025} for fixed $p$. While our result is not constant sharp, it naturally generalizes the previous bounds to settings where $p$ may also be moderately large, that is for $p \lesssim n/\log(n)$.  

When the number of graphs becomes large ($p \gtrsim n/\log(n)$), however, the threshold saturates at $\rho \gtrsim \frac{\log(n)}{n}$ and no longer decreases with $p$: the abundance of observed graphs stops to help the statistician. 

We also provide an information-theoretic lower bound for perfect recovery in the next result. 
\begin{thm}\label{thm:lowerboundinfperfect}
    For all $\eps>0$, there exists constants $c,n_0$ depending only on $\eps$, such that, if $\rho\leq c\pa{\frac{\log(np)}{n}\vee \sqrt{\frac{\log(np)}{np}}}$ and $n\geq n_0$, then $$\inf_{\widehat{\pi}}\P[\err(\widehat{\pi}, \pi^{\star})\neq 0]\geq 1-\eps\enspace,$$
    where the infimum is taken over all measurable functions $\widehat{\pi}:\mathbb{R}^{E\times p}\to \pa{\mathcal{S}_n}^p$.
    Thus, perfect recovery is impossible.
\end{thm}

Taken together, \cref{cor:MLEperfect} and \cref{thm:upperboundIT} characterize the IT threshold for perfect recovery, which is 
\begin{equation}\label{eq:IT_perfect}
    \rho\gtrsim \frac{\log(np)}{n}\vee \sqrt{\frac{\log(np)}{np}}\enspace.
\end{equation}

In regimes where \new{$\log(p)$ is much larger than $\log(n)$}, we observe a gap between the IT threshold for partial recovery \eqref{eq:IT_partial} and perfect recovery \eqref{eq:IT_perfect}. As mentioned in the introduction, this contrasts 
with the regime studied by \cite{vassaux2025} (constant $p$), where an all-or-nothing phenomenon arises: either non-trivial recovery is impossible, or perfect recovery occurs with high probability.


\section{LOW DEGREE LOWER BOUND}\label{sec:computational}

As discussed in the introduction, even when $p=2$, the performance of the best known polynomial-time algorithm \cite{ding2022polynomialtimeiterativealgorithm} does not match the optimal information-theoretic threshold \eqref{eq:IT_cas_p=2}. This suggests the existence of a statistical-computational gap for (multiple) Gaussian graphs alignment. In this section, we provide evidence of this statistical-computational gap by establishing a low degree polynomial lower bound, and we further characterize--up to logarithmic factors--the computational barrier for general $p$.

As already pointed out by \cite{even2025a}, low-degree polynomials are ill-suited to capture combinatorial constraints and therefore cannot be directly employed to construct estimators of the permutations ${\pi}^{\star}_1, \ldots, {\pi}^{\star}_p$. Instead, we concentrate on the task of estimating the \emph{node alignment matrix} $M^\star \in \mathbb{R}^{(n\times p)\times (n\times p)}$, defined as
$$ M^\star _{(u,j), (u',j')}=\1\ac{\pi^\star_j(u)=\pi^\star_{j'}(u')}\, .$$ 

The next proposition, slightly adapted from Proposition 2.1 of \cite{even2025b}, shows that computational hardness for estimating $M^\star$ for the $L^2$ loss implies computational hardness for estimating $\pi^{\star}_1,\ldots, \pi^{\star}_p$ for the $\err$ loss. We refer to \cref{prf:reductionpermutationpartnership} for a proof. 

\begin{prop}\label{prop:reductionpermutationpartnership}
    Suppose that \begin{align*}
        \MMSE_{poly}:=&\inf_{\hat{M}\, \mathrm{poly-time}} \frac{1}{p(p-1)\new{n^2}}\E\cro{\|\hat{M}-M^\star\|_F^2}\\=&\frac{1}{n}(1-\eps),
    \end{align*}
    with $0\leq \eps\leq 1$. Then, for all polynomial-time estimator $\hat{\pi}$ of $\pi^\star$, one has $$\E\cro{\pa{1-\err(\hat{\pi}, \pi^\star)}^2}\leq \sqrt{\eps}\enspace.$$
\end{prop}

Therefore, in order to provide evidence of a computational barrier for the problem of estimating $M^\star$ (and hence $\pi^\star$), we provide a low-degree polynomial lower bound for estimating $M^{\star}$. We consider the \emph{degree-$D$ minimum mean squared error} (\cite{WeinSchramm}) defined by
\begin{equation}\label{eq:MMSE}
    \MMSE_{\leq D}:=\inf_{\widehat{M}\in \mathbb{R}_D[G_{1:p}]}\frac{1}{p(p-1)n^2}\E\cro{\|\widehat{M}-M^{\star}\|^2_F},
\end{equation} where $\mathbb{R}_D[G_{1:p}]$ is the space of polynomials in the entries of $(G_1,\ldots, G_p)$ of degree at most $D$.





 We observe that the trivial estimator, $\E\cro{M^\star}$, has loss $$\frac{1}{p(p-1)\new{n^2}}\E\cro{\|\E\cro{M^\star}-M^{\star}\|^2}=\frac{1}{n}-\frac{1}{n^2} \, .$$
 Our next result characterizes a regime on which no low-degree polynomial performs significantly better than this trivial estimator. We refer to Section \ref{prf:lowdegreealignment} for a proof of this theorem.

\begin{thm}\label{thm:lowdegreealignment}
    Let $D\leq n-2$ and suppose $$\zeta:=\frac{D^3\sqrt{\rho}}{1-\sqrt{\rho}}\sqrt{1+D/2}\frac{2}{\pa{1-\frac{D+1}{n}}^2}<1 \, .$$ Then, $$\MMSE_{\leq D}\geq \frac{1}{n}-\frac{1}{n^2}-\frac{2}{\pa{n-1-D}^2}\zeta\frac{1+\zeta}{1-\zeta}\enspace.$$
\end{thm}

In particular, taking $D=O\pa{\log(np)^{1+\eta}}$ for a small $\eta>0$, we have that, whenever $\rho = o\left( \frac{1}{\log(np)^{7(1+\eta)}}\right)$, $$\MMSE_{\leq D}=\frac{1}{n}-\frac{1}{n^2}(1+o(1)) \, .$$
Since the failure of $O\pa{\log(np)^{1+\eta}}$-degree polynomial estimators is taken as evidence of hardness for polynomial-time algorithms, Theorem~\ref{thm:lowdegreealignment} suggests that estimating $M^\star$ (and hence $\pi^\star$) is computationally hard whenever
$$ \rho\leqlog 1 \, .$$

On the contrary, it is well known that as soon as $\rho$ is non-vanishing, alignment of two correlated Gaussian graphs is possible in polynomial-time (\cite{ding2022polynomialtimeiterativealgorithm}). Thus, when $\rho$ is not vanishing, a naïve strategy for aligning $p$ correlated random graphs in polynomial time would be two fix $\pi^\star_1=\id$ and to successively align every other graph with $G_1$ pairwise with the iterative algorithm developed by \cite{ding2022polynomialtimeiterativealgorithm}. 

Theorem \ref{thm:lowdegreealignment} suggests that this naïve strategy is optimal up to logarithmic factors. 
In particular, Theorem~\ref{thm:lowdegreealignment} highlights that the presence of multiple graphs does not simplify the problem computationally, and, in fact, enlarges the statistical–computational gap.

Theorem 1.1. from \cite{ding2022polynomialtimeiterativealgorithm} directly implies the following.
\begin{cor}\label{cor:upperboundpolytime}
    If $\rho > \eps$ with $\eps>0$ an independent constant, then there exists a constant $C=C(\eps)>0$ and an algorithm with $O\pa{p\times n^{C}}$-running time such that $\E\cro{\err\pa{\widehat{\pi}, \pi^\star}} \to 0$ when $n \to \infty$, that is, achieving partial recovery.
\end{cor}

\paragraph{High-level proof strategy of Theorem \ref{thm:lowdegreealignment}} In Gaussian additive models, Theorem 2.2 of \cite{WeinSchramm} provides a lower bound of the $\MMSE_{\leq D}$ with respect to a sum squared joint cumulants of the signal (see Proposition \ref{thm:schrammwein} for a statement of this formula with the notation of our model). We refer to \cite{novak2014three} for a backround on cumulants. However, for general $p$, relying only on this formula is not sufficient. Thus, we proceed in two steps:
\begin{enumerate}
    \item We consider the case $p=2$. In this case, it is sufficient to appeal to Theorem 2.2 of \cite{WeinSchramm} and to compute the cumulants involved. We point out that for analyzing those cumulants, we build upon the work of \cite{even2025a} which allows us to consider cumulants of Bernoulli variables depending only on the random permutations $\pi^\star_1, \pi^\star_2$. 
    \item We observe that the problem of aligning $p$ correlated graphs with correlation $\rho$ is harder than the problem of aligning two graphs with correlation $\sqrt{\rho}$. Plugging the first step is then sufficient for concluding the proof of the theorem for general $p$. 
\end{enumerate}

\section{DISCUSSION}




\paragraph{Comparison with $n$-dimensional feature matching.}

The problem of multiple feature matching (\cite{even2025b}) is another statistical problem where permutations are hidden in the data. Here, we observe $p$ correlated matrices $Y_1, \ldots, Y_p\in \R^{n\times n}$ where the rows of each matrix $Y_j$ is relabeled according to some $\pi^\star_j$. This model comes with a Gaussian mixture flavor :  we suppose that the rows of each matrix are drawn independently, where row $u$ of matrix $j$ is a Gaussian vector $\mathcal{N}\pa{\mu_{\pi^\star_j(u)}, I_n}$ of dimension $n$. Assume that the row relabelings $\pi^\star_1, \ldots, \pi^\star_p$ and the centers $\mu_1,\ldots, \mu_n$ are hidden. \cite{even2025b} characterizes the statistical-computational landscape of this problem, generalizing the results from \cite{Collier16} which studied only the case $p=2$.  

To clarify the comparison between the two problems, we assume that the centers $\mu_1,\ldots, \mu_n$ are taken with a gaussian prior. Then, $Y_1, \ldots, Y_p$ are correlated (non-symmetric) Gaussian matrices with $\Cov\pa{Y_{juv}, Y_{j'u'v'}}=\rho \1\ac{\pi^\star_j(u)=\pi^\star_j(u')}\1\ac{v=v'}$, for some parameter $\rho>0$ depending on the variance of the centers $\mu_1,\ldots, \mu_n$.
Heuristically, the difference between multiple feature matching and multiple gaussian graph alignment is that in the first problem, only one dimension (the rows of the matrices) is permuted, whereas in the second problem, both coordinates are permuted. One would therefore expect the second problem to be more difficult.

Yet, at the information level, the thresholds for perfect recovery and partial recovery coincide, up to some numerical constant (see \cite{even2025a}). Moreover, we observe the same exponential decay of the error above the threshold for partial recovery. With no computational constraints, both problems are equivalently difficult, but the optimal estimator is the solution to a combinatorial optimization problem over all $p$-tuple of permutations (\new{or $p'$-tuple of permutation with $p'\asymp\rho^{-1}$ in the case of Gaussian matrix alignment}), and is not computable in polynomial time for $p \geq 3$, and even for $p=2$ in Gaussian graphs alignment. 

However, Corollary 2.3 of \cite{even2025b} shows that there exists polynomial-time algorithms for multiple feature matching which exactly recover $\pi^\star$ as soon as 
$$\rho\geq_{\log} \frac{1}{\sqrt{n}}\wedge \frac{1}{\sqrt{p}} \, .$$
This suggests that the problem of multiple feature matching is much easier that the problem of multiple gaussian graph alignment where our computational barrier gives the evidence that recovering $\pi^\star$ is computationally tractable only when $\rho\geq_{\log} 1$. 

Hence, the presence of a bi-dimensional structure hardens the problem at a computational level. In our problem, the permutation of both coordinates breaks the low-rank structure of the signal and prevents any spectral method from being informative. On the contrary, \new{for large p}, in the case of multiple feature matching, the computational barrier coincides with a spectral barrier known as the \emph{BBP transition} \cite{BBP05}, below which the spectrum of the Gram matrix $YY^T$
no longer carries information about the signal.


\paragraph{Hardness for complex bi-dimensional structures.} 
The computational intractability of estimating bi-dimensional structures in matrix problems has also been observed in~\cite{even2025a}, where the authors compare the statistical–computational landscapes of high-dimensional clustering with related problems exhibiting hidden bi-dimensional structure, such as sparse clustering and bi-clustering. Without computational constraints, the full structure can be exploited, whereas polynomial-time algorithms fail to leverage the dependencies across both rows and columns simultaneously.  

This phenomenon also underlies the statistical–computational gaps in planted submatrix estimation (\cite{WeinSchramm}) and graphon estimation (\cite{luo2023computational}).  

Multiple Gaussian graph alignment thus provides another instance of a problem in which a complex hidden bi-dimensional structure gives rise to a wide statistical–computational gap.

\paragraph{Limitation: non-sharpness of the computational lower bound.} Our low-degree lower bound matches the state-of-the art polynomial-time algorithm only up to some logarithmic factors. A recent line of work succeeds in some problems to remove logarithmic factors in low-degree lower-bounds \cite{sohnwein, even2025a}, but at the price of a more subtle analysis. In \cite{even2025a}, the analysis highly rely on the independencies of the latent variables, which should not be expected to work for our problem, due to the fact that the values of the hidden permutations $\pi^\star_1,\ldots, \pi^\star_p$ are weakly dependent. We did not investigate the approach of \cite{sohnwein}.

However, to the best of our knowledge, in estimation problems where the low-degree lower bound exactly matches the performance of the best known polynomial-time algorithm, the computational barrier is characterized by a constant (e.g., the Kesten–Stigum threshold for the SBM~\cite{sohnwein}, the BBP transition for clustering~\cite{even2025a}, etc.). This leads to a sharp phase transition: the problem is efficiently solvable above the threshold, and low-degree hard below it.  

In contrast, for the problem of multiple Gaussian graph alignment, the state-of-the-art algorithm~\cite{ding2022polynomialtimeiterativealgorithm} can recover $\pi^\star$ in polynomial time as soon as the correlation $\rho$ is non-vanishing. Rather than exhibiting a sharp phase transition, the upper bound of~\cite{ding2022polynomialtimeiterativealgorithm} reveals a continuum of hardness: the number of iterations required, and hence the time complexity, grows as $\rho$ decreases. This suggests that closing the remaining (polylogarithmic) gap between the state-of-the-art polynomial-time algorithm and the low-degree lower bound is likely to be very challenging.

\paragraph{Open questions.} An intriguing direction is the alignment of correlated \ER graphs. For $p=2$, it is believed that the Otter constant characterizes the computational threshold in terms of the correlation parameter (\cite{ding2023lowdegreehardnessdetectioncorrelated, GMStree24, li25}). However, the situation for $p \geq 3$ remains unclear. Does the Otter constant continue to govern the computational barrier in this multi-graph setting, or does a different phenomenon arise? 


\bibliography{biblio(4)}

 \section*{Checklist}

 \begin{enumerate}

   \item For all models and algorithms presented, check if you include:
   \begin{enumerate}
     \item A clear description of the mathematical setting, assumptions, algorithm, and/or model. \textbf{Yes.}
     \item An analysis of the properties and complexity (time, space, sample size) of any algorithm. \textbf{Not Applicable.}
     \item (Optional) Anonymized source code, with specification of all dependencies, including external libraries. \textbf{Not Applicable.}
   \end{enumerate}

   \item For any theoretical claim, check if you include:
   \begin{enumerate}
     \item Statements of the full set of assumptions of all theoretical results. \textbf{Yes.}
    \item Complete proofs of all theoretical results. \textbf{Yes.}
     \item Clear explanations of any assumptions. \textbf{Yes.}    
   \end{enumerate}

   \item For all figures and tables that present empirical results, check if you include:
   \begin{enumerate}
     \item The code, data, and instructions needed to reproduce the main experimental results (either in the supplemental material or as a URL). \textbf{Not Applicable.}
     \item All the training details (e.g., data splits, hyperparameters, how they were chosen). \textbf{Not Applicable.}
     \item A clear definition of the specific measure or statistics and error bars (e.g., with respect to the random seed after running experiments multiple times). \textbf{Not Applicable.}
     \item A description of the computing infrastructure used. (e.g., type of GPUs, internal cluster, or cloud provider). \textbf{Not Applicable.}
   \end{enumerate}

   \item If you are using existing assets (e.g., code, data, models) or curating/releasing new assets, check if you include:
   \begin{enumerate}
     \item Citations of the creator if your work uses existing assets. \textbf{Not Applicable.}
    \item The license information of the assets, if applicable. \textbf{Not Applicable.}
     \item New assets either in the supplemental material or as a URL, if applicable. \textbf{Not Applicable.}
     \item Information about consent from data providers/curators. \textbf{Not Applicable.}
     \item Discussion of sensible content if applicable, e.g., personally identifiable information or offensive content. \textbf{Not Applicable.}
   \end{enumerate}

   \item If you used crowdsourcing or conducted research with human subjects, check if you include:
   \begin{enumerate}
     \item The full text of instructions given to participants and screenshots. \textbf{Not Applicable.}
     \item Descriptions of potential participant risks, with links to Institutional Review Board (IRB) approvals if applicable. \textbf{Not Applicable.}
     \item The estimated hourly wage paid to participants and the total amount spent on participant compensation. \textbf{Not Applicable.}
   \end{enumerate}

 \end{enumerate}

\newpage
\appendix
\thispagestyle{empty}

\onecolumn
\aistatstitle{Statistical-computational gap in multiple Gaussian graph alignment --
Supplementary Material}

\section{Proof of Theorem \ref{thm:upperboundIT}}\label{prf:upperboundIT}



\new{We will state and prove the following theorem, which will imply Theorem \ref{thm:upperboundIT}.}


\begin{thm}\label{thm:upperboundIT2}
     There exists numerical constants $C,c_1,c_2,c_3>0$ and $n_0$ such that the following holds for all $\eta \geq 1$. Whenever $\rho\geq c_1\sqrt{\frac{\log(n)+\log(\eta)}{np}}$, $p\leq C \frac{n}{\log(n)+\log(\eta)}$ and $n\geq n_0$, we have, with probability at least $1-\frac{c_2}{n^2\eta}$, $$\err(\widehat{\pi}, \pi^{\star})=0.$$   
\end{thm}

Let us first explain why Theorem \ref{thm:upperboundIT2} implies Theorem \ref{thm:upperboundIT}. Suppose that Theorem \ref{thm:upperboundIT2} is true, take $\rho> c_1\sqrt{\frac{\log(n)}{np}}$ \new{and assume $p \leq C' \rho^{-1}$ for some $C'$ to be fixed later.} Let $\eta=\exp\pa{\frac{np}{c_1^2}\rho^2-\log(n)} \geq 1$\new{, so that} $\rho= c_1\sqrt{\frac{\log(n)+\log(\eta)}{np}}$.
Since $p \leq C' \rho^{-1} = \frac{C'}{c_1} \sqrt{\frac{np}{\log(n)+\log(\eta)}}$, we have $p \leq (C'/c_1)^2 \frac{n}{\log(n)+\log(\eta)}$ and thus, taking $C' = c_1 \sqrt{C}$, Theorem \ref{thm:upperboundIT2} implies
$$\P\cro{\err(\widehat{\pi}, \pi^\star)\neq 0}\leq \frac{1}{n^2\eta}\leq \exp\pa{-cnp\rho^2}\enspace,$$ for $c=1/{c_1^2}$.

Let us now move on to the proof of Theorem \ref{thm:upperboundIT2}. In this section, for $u,u' \in [n], e,e'\in E, j,j' \in [p]$ we will use the notations $(j,e) \sim (j',e')$ (resp. $(j,e) \not\sim (j',e')$) when $\Pi^{\star}_j(e)=\Pi^{\star}_{j'}(e')$ (resp. $\Pi^{\star}_j(e) \neq \Pi^{\star}_{j'}(e')$) and $(j,u) \sim (j',u')$ (resp. $(j,u) \not\sim (j',u')$) when $\pi^{\star}_j(u)=\pi^{\star}_{j'}(u')$ (resp. $\pi^{\star}_j(u) \neq \pi^{\star}_{j'}(u')$).

We recall that deriving the MLE is equivalent to solving
\begin{equation}\label{eq:solving_B}
    \widehat{B} = \new{\argmax_{B \in \cB}} \langle GG^T, B \rangle,
\end{equation}
with $\cB$ the set of all matrices $B$ of the form $B(\pi)$ as defined in \eqref{eq:def_B}, namely matrices $B$ which can be written $B_{je, j'e'}=\frac{1}{p}\1\ac{\Pi_{j}(e)=\Pi_{j'}(e)}$ for some $\pi\in \pa{\cS_n}^p$. After solving \eqref{eq:solving_B}, the MLE can be found taking any $\widehat{\pi}$ such that $\widehat B = B(\widehat{\pi})$.

\subsection{First definitions and useful lemmas}

The main ideas for the proof of this informational upper bound are derived from \cite{even24}. The key distinction in our framework lies in the fact that, unlike in \cite{even24}, we study a quadratic form involving a vector \( G \in \mathbb{R}^{N \times p} \), rather than a quadratic form of a matrix. Consequently, the technical arguments required to derive concentration inequalities for the various terms differ from those used in \cite{even24}.

A first step is to control the error in $\pi$ in terms of the error on the normalized edge alignment matrix $B$, which looks at the permutation at the edge level. Another useful error is measured looking at the permutation at the node level. For any $\pi \in (\cS_n)^p$ we introduce its \emph{normalized vertex alignment matrix} $A(\pi) \in \R^{pn \times pn}$  as 
\begin{equation}\label{eq:def_A}
     [A(\pi)]_{ju,j'u'} = \frac{1}{p}\1_{\pi_j(u)=\pi_{j'}(u')} \, .
\end{equation}
Define $\cA := A((\cS_n)^p)$. Note that, similarly to $B$, $\pi \mapsto A(\pi)$ is a bijection from $(\cS_n)^p / \equiv$ onto $\cA$, where $\equiv$ is defined \new{in Section \ref{sec:MLE} after \eqref{eq:def_B}}. Denote $A^\star := A(\pi^\star)$. 

We start by stating a lemma that will be instrumental in the proofs to come.

\begin{lemma}\label{lem:useful} 
For all $B \in \cB$,
    \begin{equation*}
        p \new{\Tr(B^{\star}-B^{\star}B)} = \sum_{(j,e) \not\sim (j',e')} B_{je,j'e'} = \frac{1}{2} \| B^{\star}-B^{\star}B \|_1 \, .
    \end{equation*}
    \begin{equation*}
        \| B^{\star}-B^{\star}B \|_F \leq \sqrt{\| B^{\star}-B^{\star}B \|_\infty} \sqrt{\| B^{\star}-B^{\star}B \|_1} \leq \sqrt{2/p} \sqrt{\| B^{\star}-B^{\star}B \|_1} \, .
    \end{equation*}

For all $A \in \cA$,
    \begin{equation*}
        p \new{\Tr(A^{\star}-A^{\star}A)} = \sum_{\new{(j,u) \not\sim (j',u')}} A_{\new{ju,j'u'}} = \frac{1}{2} \| A^{\star}-A^{\star}A \|_1 \, .
    \end{equation*}
\end{lemma}

Another useful fact is that errors in $B$ are equivalent to  errors in $A$ for a common $\pi \in (\cS_n)^p$.

\begin{lemma}[Equivalence of errors in $A$ and $B$]\label{lem:twoerrors}
We have for all $\pi \in (\cS_n)^p$, denoting $A=A(\pi)$ and $B = B(\pi)$,
$$\frac{1}{4n}\|B^{\star}-B^{\star}B\|_1\leq \|A^{\star}-A^{\star}A\|_1\leq \frac{1}{(n-1)}\|B^{\star}-B^{\star}B\|_1\, .$$
\end{lemma}

\begin{lemma}[Errors in $\pi, A$ and $B$]
\label{lem:errordelta}
We have for all $\pi \in (\cS_n)^p$, denoting $A=A(\pi)$ and $B = B(\pi)$,
$$\err(\pi,\pi^{\star}) \leq \frac{2}{np} \times \| A^{\star} - A^{\star} A\|_1 \leq \frac{2}{n(n-1)p} \times \| B^{\star} - B^{\star} B\|_1 \, .$$
\end{lemma}

Thanks to these results, all we have to do in order to prove \cref{thm:upperboundIT} is to control \new{$\| A^{\star} - A^{\star} \widehat{A}\|_1$ with high probability, where $\widehat{A}=A(\widehat{\pi})$ with $\widehat{\pi}$ the MLE defined in \eqref{eq:MLE_on_B}}. 

\subsection{Control of \new{$\| A^{\star} - A^{\star} \widehat{A}\|_1$}}

To do so, recall that by definition,
$$ \langle GG^T, \new{\widehat{B}-B^{\star}} \rangle \geq 0, $$ which gives after plugging \eqref{eq:def_model_matrices}:
\begin{equation}\label{eq:optimUB}
    \rho \langle G^{\star}(G^{\star})^T, \new{B^{\star}-\widehat{B}} \rangle \leq 
 \sqrt{\rho(1-\rho)}\langle G^{\star}Z^T+Z(G^{\star})^T, \widehat{B} -B^{\star} \rangle + (1-\rho)\langle ZZ^T, \widehat{B} -B^{\star} \rangle \, .
\end{equation}
We will next refer to $\langle G^{\star}(G^{\star})^T, \new{B^{\star}-\widehat{B}} \rangle$ as the \emph{signal term}, $\langle G^{\star}Z^T+Z(G^{\star})^T, \widehat{B} -B^{\star} \rangle$ as the \emph{cross term}, and $\langle ZZ^T, \widehat{B} -B^{\star} \rangle$ as the \emph{quadratic term}. We will control each of these terms separately, uniformly when $\widehat B$ runs in $\cB$.

For concentration bounds involving quadratic forms of Gaussian vectors, an instrumental result is the Hanson-Wright inequality. We refer to e.g. \cite{giraud08} for a proof.

\begin{lemma}[Hanson-Wright inequality]\label{lem:HW} 
Let $d \geq 1$, $X \sim \cN(0,I_d)$ and $M$ a deterministic $d \times d$ matrix. Then, for all $x > 0$,
    \begin{equation*}
     \dP\left( |X^TMX - \Tr(M)| > \sqrt{8 \|M \|_F^2 x} \, \vee \, (8 \|M \|_{op} x) \right) \leq 2e^{-x} \, .
    \end{equation*}
\end{lemma}

Using Hanson-Wright inequality, we are able to derive the following uniform controls on the three terms in \eqref{eq:optimUB}:

\begin{proposition}\label{prop:control_B}
    \new{There exists numerical constants $C,c_1,c_2,c_3 >0$ and $n_0 \geq 1$ such that the following holds for all $\eta \geq 1$. Assume $\rho\geq c_3\sqrt{\frac{\log(n)+\log(\eta)}{np}}$, $p\leq C \frac{n}{\log(n)+\log(\eta)}$ and $n\geq n_0$.} Then, with probability at least $1-\frac{c_1}{n^2\eta}$, for all $B \in \cB$,

    \begin{itemize}
        \item \mbox{(signal term)} For some numerical constant $c>0$,
         \begin{equation*}
    \langle G^{\star}(G^{\star})^T, \new{B^{\star}-B} \rangle\geq \new{\|B^\star-B^\star\widehat{B}\|_1}/4\enspace,
    \end{equation*} 
    \item (cross term) \begin{equation*}
    |\langle G^{\star}Z^T+Z(G^{\star})^T, \widehat{B} -B^{\star} \rangle |\leq  c_2 \sqrt{\<G^\star (G^\star)^T, \new{B^{\star}-B}\>}\sqrt{\frac{\|B^\star-B^\star B\|_1}{n}\pa{\log(n)+\new{\log(\eta)}+\log\pa{\frac{\new{np}}{\|A^\star-A^\star A\|_1}}}},
    \end{equation*} 
    \item (quadratic term) for $A=A(\pi)$ where $\pi$ is such that $B=B(\pi)$,
    \begin{equation*}
        |\<ZZ^T, B-B^\star\>|\leq c_2 \|B^\star-B^\star B\|_1\pa{\frac{\log(n)+\log(\eta)+\log\left(\frac{\new{np}}{\|A^\star-A^\star A\|_1}\right)}{n}+\sqrt{\frac{\log(n)+\log(\eta)+\log\left(\frac{\new{np}}{\|A^\star-A^\star A\|_1}\right)}{np}}} \, . 
    \end{equation*}
    \end{itemize}
\end{proposition}

\subsection{Conclusion}

In the following, we restrict ourselves to the event of probability at least $1-c_1/(n^2\eta)$ on which bounds of \cref{prop:control_B} hold. First, plugging the upper bounds on the cross and quadratic terms in Equation \eqref{eq:optimUB} to $A= \widehat{A}$ defined previously and $B=\widehat{B}$ yields

\begin{align*}
    \rho \<G^\star (G^\star)^T, \new{B^\star-\widehat{B}} \>\lesssim& \sqrt{\rho}\sqrt{\<G^\star (G^\star)^T, \widehat{B}-B^\star\>}\sqrt{\frac{\|B^\star-B^\star \widehat{B}\|_1}{n}\pa{\log(n)+\log(\eta)+\log\pa{\frac{np}{\|A^\star-A^\star\widehat{A}\|_1}}}}\\
    &+\|B^\star-B^\star\widehat{B}\|_1\pa{\frac{\log(n)+\log(\eta)+\log\pa{\frac{np}{\|A^\star -A^\star \widehat{A}\|_1}}}{n}+\sqrt{\frac{\log(n)+\log(\eta)+\log\pa{\frac{np}{\|A^\star -A^\star \widehat{A}\|_1}}}{np}}}\, .
\end{align*}
\new{Discussing on which of the two terms on the right-hand side is greater, this reduces to} $$\rho\<G^\star (G^\star)^T, \widehat{B}-B^\star\>\lesssim \|B^\star-B^\star\widehat{B}\|_1\pa{\frac{\log(n)+\log(\eta)+\log\pa{\frac{np}{\|A^\star -A^\star \widehat{A}\|_1}}}{n}+\sqrt{\frac{\log(n)+\log(\eta)+\log\pa{\frac{np}{\|A^\star -A^\star \widehat{A}\|_1}}}{np}}}\enspace.$$

\new{If $\widehat{B}$ is such that $\|B^\star-B^\star\widehat{B}\|_1=0$, then $\|A^\star-A^\star\widehat{A}\|_1=0$ and $\err(\widehat{\pi},\pi^\star)=0$ by \cref{lem:errordelta}, and there is nothing to prove. If $\|B^\star-B^\star\widehat{B}\|_1>0$,} the lower bound on the signal term in \cref{prop:control_B} entails

\begin{equation*}
    \rho \lesssim \frac{\log(n)+\log(\eta)+\log\pa{\frac{np}{\|A^\star -A^\star \widehat{A}\|_1}}}{n}+\sqrt{\frac{\log(n)+\log(\eta)+\log\pa{\frac{np}{\|A^\star -A^\star \widehat{A}\|_1}}}{np}}\enspace.
\end{equation*}

Under our assumptions, $\rho\geq c_3\sqrt{\frac{\log(n)+\log(\eta)}{np}}$ and $p\leq C \frac{n}{\log(n)+\log(\eta)}$, so $\frac{\log(n)+\log(\eta)}{n} \leq \frac{\sqrt{C}}{c_3}\rho$. Thus, if $c_3$ is large enough, the previous inequality yields

$$\rho \lesssim \frac{\log\pa{\frac{np}{\|A^\star -A^\star \widehat{A}\|_1}}}{n}+\sqrt{\frac{\log\pa{\frac{np}{\|A^\star -A^\star \widehat{A}\|_1}}}{np}}\enspace,$$

which leads to the existence of some numerical constants $c',c$ satisfying 



$$
\|A^\star-A^\star \widehat{A}\|\leq np \exp\pa{-c' n\pa{\rho \wedge p\rho^2}}\leq  np \exp\pa{-c np\rho^2}\, . 
$$

Using Lemma \ref{lem:errordelta}, we deduce that $$\err(\widehat{\pi},\pi^\star)\leq 2\exp\pa{-c np\rho^2}\enspace.$$
When $\rho \geq c_3 \sqrt{\frac{\log(n)+\log(\eta)}{np}}$ with $c_3$ large enough, the error drops \new{strictly below $\frac{1}{n^2} \leq \frac{1}{np}$, and thus equals $0$}. This concludes the proof of \cref{thm:upperboundIT2}. 

\subsection{Proof of \cref{prop:control_B}}

We now give the proof of \cref{prop:control_B}. We fix $\eta\geq 1$, which may depend on $n$. Throughout this proof, for brevity, we will denote
$$\delta_B := \| B^{\star}-B^{\star}B \|_1 \, .$$  

\proofstep{Step 1: the signal term.}
We have
\begin{align*}
    \langle G^{\star}(G^{\star})^T, \new{B^{\star}-B} \rangle & = (G^{\star})^T(\new{B^{\star}-B})G^{\star},
\end{align*} and $\Var(G^{\star})=pB^\star$. By Hanson-Wright inequality (\cref{lem:HW}), with probability $1-2e^{-x}$,
\begin{align*}
    \langle G^{\star}(G^{\star})^T, \new{B^{\star}-B}  \rangle & \geq p\Tr(B^{\star}(\new{B^{\star}-B} )) - \sqrt{8 p^2\|B^{\star}(\new{B^{\star}-B} )B^\star \|_F^2 x} \, \vee \, (8 p\|B^{\star}(\new{B^{\star}-B}) B^\star\|_{op} x) \\
    & = \frac{1}{2} \| B^{\star}-B^{\star}B \|_1 - \sqrt{8 p^2\|B^{\star}(\new{B^{\star}-B}) B^\star\|_F^2 x} \, \vee \, (8 p\|B^{\star}(\new{B^{\star}-B}) B^\star\|_{op} x) ,
\end{align*} 
where \cref{lem:useful} justifies the equality. \new{The symmetry of $B$ and $B^\star$ together with \cref{lem:useful} yields $\|B^{\star}(B^{\star}-B)B^\star \|_{F} = \|B^{\star}(B^{\star}-B) \|_{F} \leq \sqrt{2\delta_B/p}.$ Moreover, $B$ and $B^\star$ are projections, thus $\|B^{\star}(\new{B^{\star}-B} )B^\star \|_{op} \leq 2$. Besides, one always have $\|B^{\star}(B^{\star}-B)B^\star \|_{op} \leq \|B^{\star}(B^{\star}-B)B^\star \|_{F}$.}
Thus,
\begin{align*}
    \langle G^{\star}(G^{\star})^T, \new{B^{\star}-B} \rangle & \geq {\delta_B}/{2} - \left( (4 \sqrt{p\delta_B x}) \, \vee \, (16p(1\wedge \sqrt{{\delta_B}/{p}})x) \right) \, .
\end{align*}

Let, for all $1 \leq t \leq 2np$, 
\begin{equation}
    \label{eq:def_Bt}
    \mathcal{B}_t:= \ac{B\in \mathcal{B},\quad \delta_B \in \left](t-1)\frac{n-1}{2}, t\frac{n-1}{2} \right]}\, .
\end{equation}

Now, our strategy to bound $ \langle G^{\star}(G^{\star})^T, B -B^{\star} \rangle $ uniformly on $\cB_t$ and then do a union bound on $1 \leq t \leq 2np$. For $B \in \cB_t$ for a fixed $t$, we will apply the previous inequality for \new{$x = x_t \gtrsim \log 2 + \log|\cB_t| + \log(n^3p)+\log(\eta)$} so that $2|\cB_t|e^{-x} \leq \frac{1}{n^3p\eta}$. Elementary combinatorial arguments enable us to bound $|\cB_t|$.

\begin{lemma}[Control of the size of $\cA_t$ and $\cB_t$]\label{lem:control_At_Bt}
Assume $n \geq 2$. Let, for all $1 \leq t \leq 2np$,
$$ \cA_t:= \ac{A\in \mathcal{A},\quad \|A^{\star}-A^{\star}A\|_1\in \left[\frac{t-1}{2},\frac{t}{2} \right]}\, .$$ For all $1 \leq t \leq 2np$,
$$ |\cA_t| \leq \binom{np}{t \wedge np} n^{t \wedge np} \quad \mbox{and} \quad |\cB_t| \leq t \left(\frac{16n^2pe}{ t \wedge np}\right)^{t \wedge np} \, . $$
In turn,
$$ \log |\cB_t| \lesssim t\log n +   t \log \left(\frac{16npe}{t} \right) \, . $$ 

\end{lemma}

According to \cref{lem:control_At_Bt}, for $1 \leq t \leq 2np$, it is enough to take $x = x_t \asymp t\log(n)+t\log\pa{{2np}/{t}}+\log(n^3p)+\log(\eta)$, where $\asymp$ denotes equality up to multiplicative constant, to get that with probability at least $1-\frac{1}{n^3p\eta}$, for all $B \in \cB_t$, 
\begin{align*}
     \langle G^{\star}(G^{\star})^T, \new{B^{\star}-B} \rangle \geq& \frac{\delta_B}{2}- (4 \sqrt{p\delta_B x_t}) \, \vee \, (16p\pa{1 \wedge \sqrt{{\delta_B}/{p}}}x_t) )
    \\
    \geq& \frac{\delta_B}{2}- c\pa{\sqrt{p\delta_B\pa{t\log(n)+t\log\pa{{2np}/{t}}+\log(n^3p)+\log(\eta)}}}\\
    & \quad \quad \quad -cp\pa{t\log(n)+t\log\pa{{2np}/{t}}+\log(n^3p)+\log(\eta)},
\end{align*}
with $c$ some numerical constant, \new{where we used $a \vee (b \wedge c) \leq a+c$ in the second line}. We have, keeping in mind that $t \asymp \delta_B/n \asymp \|A^\star-A^\star A\|_1$ (by definition \eqref{eq:def_Bt} and \cref{lem:twoerrors}) and that when $B\neq B^\star$ (resp. $A\neq A^\star$), $\delta_B\gtrsim n$ (resp. $\|A^\star-A^\star A\|_1\gtrsim 1$);
\begin{itemize}
    \item $pt\log(n)\lesssim  \frac{\delta_B}{n}p \log (n)\leq {\delta_B}/{256}$, since we assumed $p\leq C \frac{n}{\log(n)}$ taking $C$ small enough;
    \item $pt\log\pa{{2np}/{t}}\lesssim \frac{\delta_B}{n} p \log(2np)\leq \delta_B/256$, for the same reasons;
    \item $p\log(n^3p)\lesssim \frac{\delta_B }{n\|A^\star-A^\star A\|_1}p\log(n^3p)\leq \delta_B/{256}$, for the same reasons;
    \item $p\log(\eta)\lesssim \frac{\delta_B}{n} p\log(\eta)\leq \delta_B/{256}$, since we assumed $p\leq C \frac{n}{\log(\eta)}$, taking $C$ small enough;
    
\end{itemize}
We deduce that, for some fixed $1 \leq t \leq 2np$, under our assumptions on $p$ and $\rho$, with probability at least $1-\frac{1}{n^3p\eta}$,
\begin{equation*}
    \forall B \in \cB_t, \quad \langle G^{\star}(G^{\star})^T, \new{B^{\star}-B} \rangle\geq \delta_B/4\enspace.
\end{equation*}
We conclude the proof of the control of the signal term with an union bound on $1 \leq t \leq 2np$. 



\proofstep{Step 2: the cross term}
For $B \in \cB$, the cross term writes
\begin{align*}
    \langle G^{\star}Z^T+Z(G^{\star})^T, \widehat{B} -B^\star \rangle & = \sum_{(j,e), (j',e')} (G^\star_{je} Z_{j'e'}+G^\star_{j'e'}Z_{je})(B_{je,j'e'}-B^\star_{je,j'e'})\\
    & = \sum_{(j,e), (j',e')} (G^\star_{je} -G^\star_{j'e'})(Z_{j'e'}-Z_{je})(B_{je,j'e'}-B^\star_{je,j'e'}) \\
    & = \sum_{(j,e), \not\sim (j',e')} (G^\star_{je} -G^\star_{j'e'})(Z_{j'e'}-Z_{je})B_{je,j'e'},
\end{align*} where the second equality comes from the identity $B\1 = B^{\star}\1 = \1$, and the last equality comes from the fact that $G^\star_{je} = G^\star_{j'e'}$ when $(j,e) \sim (j',e')$. Since the previous quantity can be split into two terms which are equal in distribution, it is enough to control, for all $B \in \cB$, $$ \sum_{(j,e) \not\sim (j',e')} (G^\star_{je} -G^\star_{j'e'})Z_{je}B_{je,j'e'} \, .$$
Again, we proceed by bounding this linear form uniformly for all $B\in \mathcal{B}_t$, where $\cB_t$ is defined in \eqref{eq:def_Bt}, and we then apply a union bound over $t\in [1,2np]$. Let us first fix $t\in [1,2np]$ and $B\in \mathcal{B}_t$. Then, conditioning on the signal $G^\star$, $\sum_{(j,e) \not\sim (j',e')} (G^\star_{je} -G^\star_{j'e'})Z_{je}B_{je,j'e'}$ is a \new{centered} Gaussian random variable with variance given by
\new{
\begin{align*}
    \sum_{(j,e)}\pa{\sum_{(j',e') \not\sim (j,e)}\pa{G^\star_{je}-G^\star_{j'e'}}B_{je, j'e'}}^2 & =  \sum_{(j,e)}\pa{\sum_{(j',e')}\pa{G^\star_{je}-G^\star_{j'e'}}B_{je, j'e'}}^2 \\
    & \leq \sum_{(j,e)} \sum_{(j',e')}\pa{G^\star_{je}-G^\star_{j'e'}}^2B_{je, j'e'} \\
    & =  \sum_{(j,e)} \sum_{(j',e')}\pa{G^\star_{je}-G^\star_{j'e'}}^2(B_{je, j'e'} -B^{\star}_{je, j'e'})\\
     & = -2\sum_{(j,e)} \sum_{(j',e')}G^\star_{je}G^\star_{j'e'}(B_{je, j'e'} -B^{\star}_{je, j'e'}) = 2 \langle G^\star (G^\star)^T, B^\star-B \rangle ,
\end{align*}
where the first inequality comes from Jensen's inequality, since for all $(j,e)$, $\sum_{(j',e')}B_{je,j'e'}=1$, third line comes from the fact that $G^\star_{je}=G^\star_{j'e'}$ whenever $B^{\star}_{je, j'e'}=1$, and the last line is a consequence of the identity $B\1 = B^{\star}\1 = \1$.} 

Standard Gaussian concentration gives that, with probability at least $1-2e^{-x}$,

$$\left|\sum_{(j,e) \not\sim (j',e')} (G^\star_{je} -G^\star_{j'e'})Z_{je}B_{je,j'e'}\right|\lesssim \sqrt{\<G^\star (G^\star)^T, \new{B^\star-B}\>}\sqrt{x}$$

Taking again $x = x_t \asymp t\log(n)+t\log\pa{{2np}/{t}}+\log(n^3p)+\log(\eta)$, we have, with probability at least $1-\frac{1}{n^3p\eta}$, for all $B\in \mathcal{B}_t$, 

\begin{align*}
    |\sum_{(j,e) \not\sim (j',e')} (G^\star_{je} -G^\star_{j'e'})Z_{je}B_{je,j'e'}|\lesssim & \sqrt{\<G^\star (G^\star)^T, \new{B^\star-B}\>}\sqrt{t\log(n)+t\log\pa{{2np}/{t}}+\log(n^3p)+\log(\eta)}\\
    \leq & \sqrt{\<G^\star (G^\star)^T, \new{B^\star-B}\>}\sqrt{\frac{\delta_B}{n}\pa{\log(n)+\log(\eta)+\log\pa{\frac{np}{\|A^\star-A^\star A\|_1}}}},
\end{align*}
where the last inequality comes from the fact $p\leq C \frac{n}{\log(n)}$ and that $t \asymp \delta_B/n \asymp \|A^\star-A^\star A\|_1$ (by definition \eqref{eq:def_Bt} and \cref{lem:twoerrors}). 
We conclude the proof of the control of the cross term with an union bound on $1 \leq t \leq 2np$.

\proofstep{Step 3: the quadratic noise term.}

We seek to control $\<ZZ^T, B^\star-B\>$ for $B\in \mathcal{B}$. We shall suppose without loss of generality that $\pi^\star=\id$. To any $B\in \mathcal{B}$, we identify $B=B(\pi)$ with $\pi\in (\cS_n)^p$ minimizing $\mathrm{der}(\pi):=\sum_{j\in [p]}\sum_{u\in [n]}\1\ac{\pi_j(u)\neq u}$.  From Lemma \ref{lem:errordelta}, we know that $\mathrm{der}(\pi)=\err(\pi, \pi^\star)\lesssim \frac{1}{n} \|B^\star-B^\star B\|_1$.


As previously, our strategy is to upper-bound the quadratic term uniformly on all $\mathcal{B}_t$ and then use the union bound on $1 \leq t \leq 2np$. First, we shall fix such a $1 \leq t \leq 2np$ and $B\in \mathcal{B}_t$. Using Hanson-Wright Lemma (Lemma \ref{lem:HW}), gives that, with probability at least $1-2e^{-x}$, 
$$|\<ZZ^T, B-B^\star\>|\lesssim \|B-B^\star\|_{op}x\vee \|B-B^\star\|_F\sqrt{x}\, .$$
Bounds on $\|B-B^\star\|_{op}$ and $\|B-B^\star\|_F$ are given in the following Lemma. 
\begin{lemma}\label{lem:controlnormBstar-B}
  $$\|B-B^*\|_{F}\leq \sqrt{{\delta_B}/{p}} \quad \mbox{and} \quad \|B-B^\star\|_{op}\leq 1\, .$$
\end{lemma}
Thus, with probability at least $1-2e^{-x}$, we have $$|\<ZZ^T, B-B^\star\>|\leq  x\vee \sqrt{{x\delta_B}/{p}}\, .$$

Doing an union bound, we get that, with probability at least $1-2|\mathcal{B}_t|e^{-x}$, uniformly over all $B\in \mathcal{B}_t$, $$|\<ZZ^T, B-B^\star\>|\leq x\vee \sqrt{{x\delta_B}/{p}}\, .$$ As before, we choose $x = x_t \asymp t\log(n)+t\log\pa{{2np}/{t}}+\log(n^3p)+\log(\eta)$ so that, with probability at least $1-\frac{1}{n^3p\eta}$, for all $B\in \mathcal{B}_t$, uniformly over all $B\in \mathcal{B}_t$, we have 
\begin{multline*}
|\<ZZ^T, B-B^\star\>| \lesssim  \pa{t\log(n)+t\log\pa{{2np}/{t}}+\log(n^3p)+\log(\eta)} \\
+\sqrt{\frac{\delta_B}{p}\pa{t\log(n)+t\log\pa{{2np}/{t}}+\log(n^3p)+\log(\eta)}}\, .
\end{multline*}

We remark that, as in step 1, keeping in mind that $t \asymp \delta_B/n \asymp \|A^\star-A^\star A\|_1$ (by definition \eqref{eq:def_Bt} and \cref{lem:twoerrors}) and that when $B\neq B^\star$ (resp. $A\neq A^\star$), $\delta_B\gtrsim n$ (resp. $\|A^\star-A^\star A\|_1\gtrsim 1$);
\begin{itemize}
\item $t\log(n)\lesssim \frac{\delta_B}{n}\log(n)$;
\item $t\log(2np/t)\lesssim \frac{\delta_B}{n}\log(\frac{np}{\|A^\star-A^\star A\|_1})$;
\item $\log(n^3p) \lesssim t\log(n)\lesssim \frac{\delta_B}{n}\log(n)$, since we assumed $p\leq C \frac{n}{\log(n)}$;
\item $\log(\eta)\lesssim \frac{\delta_B}{n}\log(\eta)$.
\end{itemize}

Finally, we end up with 
$$|\<ZZ^T, B-B^\star\>|\lesssim \frac{\delta_B}{n}\pa{\log(n)+\log(\eta)+\log\pa{\frac{np}{\|A^\star-A^\star A\|_1}}}+\delta_B \sqrt{\frac{1}{np} \pa{\log(n)+\log(\eta)+\log\pa{\frac{np}{\|A^\star-A^\star A\|_1}}}}\, .$$
We conclude the proof of the control of the quadratic noise term with an union bound on $1 \leq t \leq 2np$. 


\subsection{Further proofs of \cref{{prf:upperboundIT}}}\label{sec:further_proofs_A}

\subsubsection{Proof of \cref{lem:useful}}
For all $B \in \cB$, note that 
$$(B^{\star}-B^{\star}B)_{je,j'e'} = \begin{cases}
    \frac{1}{p}-\frac{1}{p} \sum_{(k,f): (k,f)\sim(j,e)} B_{kf,j'e'} & \mbox{if } (j,e) \sim (j',e') \\
    -\frac{1}{p} \sum_{(k,f):(k,f)\sim(j,e)} B_{kf,j'e'} & \mbox{if } (j,e) \not\sim (j',e') 
\end{cases}$$
Thus, 
\begin{align*}
    p \Tr(\new{B^{\star}-B^{\star}B}) & = p \sum_{(j,e)} \left( \frac{1}{p}-\frac{1}{p} \sum_{(k,f): (k,f)\sim(j,e)} B_{kf,je}\right) \\
    & = pN - \left( pN -  \sum_{(k,f):(k,f)\not\sim(j,e)} B_{kf,j'e'} \right) \\
    & = \sum_{(j,e) \not\sim (j',e')} B_{je,j'e'} \, .
\end{align*}
On the other hand, 
\begin{align*}
    \| B^{\star}-B^{\star}B \|_1 & = \sum_{(j,e)} \left[\sum_{(j',e'): (j',e')\sim(j,e)}\left( \frac{1}{p}-\frac{1}{p} \sum_{(k,f): (k,f)\sim(j,e)} B_{kf,j'e'}\right) + \sum_{(j',e'): (j',e')\not\sim(j,e)} \frac{1}{p} \sum_{(k,f):(k,f)\sim(j,e)} B_{kf,j'e'} \right]  \\
    & = \frac{1}{p} \sum_{(j,e)}  
    \sum_{(k,f):(k,f)\sim(j,e)}
    \left[\sum_{(j',e'): (j',e')\sim(j,e)} \left(\frac{1}{p}-B_{kf,j'e'}\right) + \sum_{(j',e'): (j',e')\not\sim(j,e)} B_{kf,j'e'} \right]  \\
    & = \frac{1}{p} \sum_{(j,e)}  
    \sum_{(k,f):(k,f)\sim(j,e)} \left[ 1 - 
     \left( 1 - \sum_{(j',e'): (j',e')\not\sim(j,e)} B_{kf,j'e'}\right) + \sum_{(j',e'): (j',e')\not\sim(j,e)} B_{kf,j'e'}  \right] \\
      & = \frac{2}{p} \sum_{(j,e)}  
    \sum_{(k,f):(k,f)\sim(j,e)}  \sum_{(j',e'): (j',e')\not\sim(j,e)} B_{kf,j'e'}  \\
     & = \frac{2}{p} \sum_{(k,f)}  
    \sum_{(j,e):(j,e)\sim(k,f)}  \sum_{(j',e'): (j',e')\not\sim(j,e)} B_{kf,j'e'}  \\
    & = 2 \sum_{(k,f)} \sum_{(j',e'): (j',e')\not\sim(j,e)} B_{kf,j'e'} = 2 \sum_{(j,e) \not\sim (j',e')} B_{je,j'e'} \, .
\end{align*}
Then, the last inequality comes from the fact that since $B^\star$ is stochastic, $$ \| B^{\star}-B^{\star}B \|_\infty = \| B^{\star}(B^{\star}-B) \|_\infty \leq \| B^{\star}-B \|_\infty \leq 2/p \, . $$

The proof works similarly for all $A \in \cA$.

\subsubsection{Proof of Lemma \ref{lem:twoerrors}}
    Assume without loss of generality that $\pi^\star_j = \id$ for all $j \in [p]$. By \cref{lem:useful},
    $$\|B^{\star}-B^{\star}B\|_1=\frac{2}{p}\sum_{e\neq e'}\sum_{j\neq j'}\1\ac{\Pi_j(e)=\Pi_{j'}(e')}$$
    and $$\|A^{\star}-A^{\star}A\|_1=\frac{2}{p}\sum_{u\neq u'}\sum_{j\neq j'}\1\ac{\pi_j
    (u)=\pi_{j'}(u')}\,.$$
    Let us fix $j,u$ and $j',u'$ such that $u \neq u'$ and  $\pi_j(u)=\pi_{j'}(u')$. For all $v \notin \set{u,u'} $, we have $\Pi_j\pa{\ac{u,v}}=\Pi_{j'}\pa{\ac{u', \pi_{j'}^{-1}\circ \pi_j(v) }}$, with $\ac{u,v}\neq \ac{u', \pi_{j'}^{-1}\circ \pi_j(v) }$. Thus, the number of edge pairs $((j,e), (j',e'))$ with $e\neq e'$, such that $u$ is a node of $e$, $u'$ is a node of $e'$, and $\Pi_j(e)=\Pi_{j'}(e')$, is $|[n]\setminus \set{u,u'}|$ which is either $n-1$ or $n$. Moreover, each such pair $((j,e), (j',e'))$ can be obtained by at most four distinct choices of $u,u'$. We conclude that $$\frac{1}{4n}\|B^{\star}-B^{\star}B\|_1\leq \|A^{\star}-A^{\star}A\|_1\leq \frac{1}{(n-1)}\|B^{\star}-B^{\star}B\|_1\, .$$

\subsubsection{Proof of \cref{lem:errordelta}}
In light of Lemma \ref{lem:twoerrors}, it is sufficient to prove the first inequality. This inequality is a direct consequence of Lemma 9 of \cite{even24}. \new{Since our framework and our notation are differ from the one of \cite{even24}, let us nonetheless prove Lemma \cref{lem:errordelta}. Let $\pi=\pi_1,\ldots, \pi_p$. Our goal is to bound the proportion of misclassified points $$\err(\pi,\pi^{\star}) := \min_{\psi \in\mathcal{S}_{n}}\frac{1}{np}\sum_{u=1}^{n}\sum_{j=1}^{p}\1\ac{\psi(\pi_j(u))\neq \pi^{\star}_j(u)}\enspace ,$$ with respect to $\|A^\star-A^\star A\|_1$. For that purpose, let us chose some relevant permutation $\psi\in \mathcal{S}_n$ such that $$\frac{1}{np}\sum_{u=1}^{n}\sum_{j=1}^{p}\1\ac{\psi(\pi_j(u))\neq \pi^{\star}_j(u)}\leq \frac{2}{np}\times \|A^\star-A^\star A\|_1\enspace.$$ Let us suppose without loss of generality that $\pi^\star=id$.

\begin{lemma}\label{lem:choiceofpsi}
    There exists a permutation $\psi\in \mathcal{S}_n$ satisfying; for all $u\neq v\in [n]$, $|\ac{j\in [p], \psi\circ \pi_j(u)=v}|\leq p/2$.
\end{lemma}

Let $\psi\in\mathcal{S}_n$ given by Lemma \ref{lem:choiceofpsi}. Without loss of generality, we suppose that $\psi=id$. By definition, we have $$\err(\pi,\pi^\star)\leq \frac{1}{np}\sum_{u=1}^n\sum_{j=1}^p\1\ac{\pi_j(u)\neq u}=\frac{1}{np}\sum_{u\neq v}|\ac{j\in [p], \pi_j(u)=v}|\enspace.$$

Let us fix $u\neq v$. By Lemma \ref{lem:choiceofpsi}, we have $|\ac{j\in [p], \pi_j(u)=v}|\leq \frac{p}{2}$. Hence, all elements $(j,u)$ with $\pi_j(u)=v$ satisfy $\sum_{j'\neq j}\sum_{v'\neq u}A_{ju, j'v'}\geq \frac{1}{2}$. We arrive at

\begin{align*}
    \frac{1}{np}\sum_{u\neq v}|\ac{j\in [p], \pi_j(u)=v}|=&\frac{1}{np}\sum_{u\neq v}|\ac{j\in [p], \pi_j(u)=v}|\\
    \leq& \frac{1}{np}\sum_{u\neq v}2\sum_{j\in [p], \pi_j(u)=v}\sum_{v'\neq u}A_{ju, j'v'}\\
    \leq&\frac{2}{np}\sum_{u\neq v'}\sum_{j\neq j'}A_{ju,j'v'}\\
    \leq&\frac{2}{np}\|A^\star-A^\star A\|_1\enspace.
\end{align*}

This concludes the proof of the lemma.

\begin{proof}[Proof of Lemma \ref{lem:choiceofpsi}]
Let us construct $\phi=\psi^{-1}$. For $u,v$ with $|\ac{j\in [p], \pi_j(u)=v}|>p/2$, let us take $\phi(v)=\pi_j(u)$. This is an bijection from all the elements $v$ such that there exists $u$ with $|\ac{j\in [p], \pi_j(u)=v}|>p/2$ to a subset of $[n]$. This bijection can be extanded to a bijection $\phi:[n]\to [n]$. The choice $\psi=\phi^{-1}$ satisfies the condition of the lemma. 
\end{proof}}

\subsubsection{Proof of \cref{lem:control_At_Bt}}
Fix $1 \leq t \leq 2np$, $A \in \cA_t$ and $\pi \in (\cS_n)^p$ such that $A = A(\pi)$. Without loss of generality, assume that $\psi=\id$ reaches the minimum in the definition of $\err(\pi,\pi^\star)$.

Remark that since $A \in \cA_t$, thanks to \cref{lem:errordelta},
$$ \frac{t}{2} \geq \|A^{\star}-A^{\star}A\|_1 \geq \frac{np}{2} \err(\pi,\pi^\star) = \frac{1}{2}|\set{(j,u) \in [p] \times [n], \pi_j(u) \neq \pi^\star_j(u)}|,$$ 
thus, we have at most $t$ pairs $(j,u)$ such that $\pi_j(u) \neq \pi^\star_j(u)$. Since this number is also upper bounded by $np$, we have at most $t\wedge np$ such pairs. Choosing these pairs among the $np$ choices, and choosing values $\pi_j(u)$ for each of these $(j,u)$ among $n$ possibilities gives $$ |\cA_t| \leq \binom{np}{t \wedge np} n^{t \wedge np} \, . $$

Then, by \cref{lem:twoerrors}, if $B(\pi)$ is an element of $\cB_t$, then $A(\pi)$ lies in some $A_{s}$ with $t/16 \leq s \leq t$ as soon as $n \geq 2$. This gives
$$|\cB_t| \leq \sum_{s=t/16}^{t}\binom{np}{s \wedge np} n^{s \wedge np} \leq \sum_{s=t/16}^{t}\left(\frac{npe}{s \wedge np}\right)^{s \wedge np} n^{s \wedge np} \leq t \left(\frac{16n^2pe}{ t \wedge np} \right)^{t \wedge np} \, . $$
\new{This gives in turn that 
\begin{align*}
    \log|\cB_t| & \lesssim \log t + (t \wedge np) \log \left( \frac{16n^2pe}{t \wedge np} \right) \\
    & \leq \log(2np) + t \log \left( \frac{16n^2pe}{t} \right) + t \log \left( 16ne \right) \\
    & \lesssim t \log n + t \log \left( \frac{16npe}{t} \right) .
\end{align*}}

\subsubsection{Proof of Lemma \ref{lem:controlnormBstar-B}}


Since both $B$ and $B^\star$ are projection matrices, we directly deduce $\|B-B^\star\|_{op} \leq 2$. For controlling $\|B-B^\star\|_F$, we proceed by counting the number of non-zero entries of $B-B^\star$, combined with the fact that $\|B-B^\star\|_{\infty}=\frac{1}{p}$. Let $e=\ac{u,v}$, $e'=\ac{u',v'}$ and $j,j'$. 
 Denote
 $$ \mathcal{C}_j := \set{u \in [n], \pi_j(u) \neq \pi^\star_j(u)} \, .$$

\new{In order to have} $\pa{B-B^\star}_{je, j'e'}\neq 0$, we need to have either:
\begin{enumerate}
    \item $e=e'$ and either $u\in \mathcal{C}_j$, $v\in \mathcal{C}_j$, $u'\in \mathcal{C}_{j'}$ or $v'\in \mathcal{C}_{j'}$. The number of such pairs $(j,e), (j',e')$ is at most $2np\sum_{j\in [p]}|\mathcal{C}_j|=2n^2p^2\times \err(\pi,\pi^\star)\lesssim p\|B^\star-B^\star B\|_1$;
    \item $e\neq e'$ but we must have $\ac{\Pi_{j}(e)=\Pi_{j'}(e')}$ and, a fortiori, either $u\in \mathcal{C}_j$, $v\in \mathcal{C}_j$, $u'\in \mathcal{C}_{j'}$ or $v'\in \mathcal{C}_{j'}$. The number of such pairs is at most $2np\times \err(\pi,\pi^\star)\lesssim p\|B^\star-B^\star B\|_1$.
\end{enumerate}

We can conclude with $\|B-B^\star\|_F\lesssim \sqrt{{\delta_B}/{p}}$.

\section{Proof of Theorem \ref{thm:upperboundITgeneralp}}\label{prf:upperboundITgeneralp}

Assume $p\geq  C\rho^{-1}$ and $\rho\geq c_1\pa{\frac{\log(n)}{n}\vee \sqrt{\frac{\log(n)}{np}}}$ for some large enough numerical constant $c_1$. Take $p'=C\rho^{-1}$, where $C$ is the constant given in Theorem \ref{thm:upperboundIT}. \new{It is clear that, for this new $p'$, $$\rho\geq c_1 \frac{\log n}{n} \implies \rho^2 \geq c_1 \rho \frac{\log (n)}{n} = c_1 C \frac{\log (n)}{np'}  \implies \rho\geq \sqrt{c_1 C}\sqrt{\frac{\log(n)}{p'n}}\enspace.$$}

We denote in the following $G'$ and $\widehat G'$ the graphs defined by $$G'_{e}=\frac{1}{p'}\sum_{j\leq p'}G_{j \pa{\Pi_j^\star}^{-1}(e)} \quad \mbox{and} \quad \widehat G'_{e}=\frac{1}{p'}\sum_{j\leq p'}G_{j \pa{\widehat\Pi_j}^{-1}(e)}\enspace.$$

Using Theorem \ref{thm:upperboundIT}, we deduce that the MLE applied to $G_1, \ldots, G_{p'}$ achieves perfect recovery with high probability.

\begin{cor}\label{cor:MLEp'}
    Suppose $\rho\geq c\pa{\frac{\log(n)}{n}+\sqrt{\frac{\log(n)}{np}}}$ with $c$ large enough and $n\geq n_0$ with $n_0$ some numerical constant. Then, $$\P\pa{\err\pa{\pa{\widehat{\pi}_j}_{j\leq p'},\pa{\pi_j^\star}_{j\leq p'} }\neq 0}\leq \exp\pa{-c'n\rho}\enspace,$$ with $c'$ some numerical constant.
\end{cor}

On the event of high probability of Corollary \ref{cor:MLEp'}, the graphs $G'$ and $\widehat{G'}$ are isomorphic. Thus, it is sufficient to analyse the alignment of a given graph $G_j$, for $j>p'$, with the graph $G'$. We remark that the edge correlation between two such graphs is equal to $$\rho' = \frac{\sqrt{p'-1}\,\rho}{\sqrt{1 + (p'-2)\rho}}\gtrsim  \sqrt{\rho}\geq c \sqrt{\frac{\log(n)}{n}}\enspace,$$ with $c$ some numerical constant that can be taken arbitrary large. Hence, we can apply Theorem \ref{thm:upperboundIT} for the specific case $p=2$.

Formally, we compute $\widehat{\pi}_j$ the unique element such that $(\widehat{\pi}_1,\widehat{\pi}_j)$ is the MLE of $(\widehat{G'}, G_j)$. Let $\widehat{\pi}_j^{\mathrm{oracle}}$ the unique element such that $(\pi^\star_1,\widehat{\pi}_j^{\mathrm{oracle}})$ is the MLE of $(G', G_j)$. Next corollary controls the error of $\widehat{\pi}_j^{\mathrm{oracle}}$.

\begin{cor}\label{cor:aligningtwographs}
    There exists numerical constants $c,c',n_0$ such that the following holds for all $j>p'$ whenever $n\geq n_0$ and $\rho'\geq c \sqrt{\frac{\log(n)}{n}}$. The MLE applied to $G'$ and $G_j$ satisifies $$\P\cro{\delta_H(\widehat{\pi}_j^{\mathrm{oracle}},  \pi_j^\star)\neq 0}\leq \exp\pa{-c'n\rho}\enspace,$$ with $\delta_H$ the Hamming distance. \new{We denote $\mathcal{E}_j$ the event in the LHS above.}
\end{cor}

If $\rho\geq c\pa{\frac{\log(n)}{n}+\sqrt{\frac{\log(n)}{np}}}$ for some large enough numerical constant $c$, then Corollary \ref{cor:aligningtwographs} and Corollary \ref{cor:MLEp'} hold.

Then, let us restrict ourselves to the event \new{$\cA$ on which} $\err\pa{\pa{\widehat{\pi}_j}_{j\leq p'},\pa{\pi_j^\star}_{j\leq p'}} = 0$. This means that there exists $\psi\in \cS_n$ such that $\widehat{\pi}_j=\psi\circ \pi^\star_j$ for all $j\leq p'$.  Then, for all $j>p'$, we have $\widehat{\pi}_j=\psi\circ \widehat{\pi}_j^{\mathrm{oracle}}$. We deduce from this the upper-bound \new{on event $\cA$} $$\err(\widehat{\pi}, \pi^\star)\leq  \1\ac{\err\pa{\pa{\widehat{\pi}_j}_{j\in [p']}, \pa{\pi^\star_j}_{j\in [p']}}\neq 0}+\frac{1}{p}\sum_{j>p'}\1\pa{\mathcal{E}_j}\enspace.$$

Taking to the expectation, we deduce \new{$$\E\cro{\err(\widehat{\pi}, \pi^\star)}\leq \dP(\cA) + \exp\pa{-cn\rho}\enspace \leq \exp\pa{-c_2n\rho},$$} for some numerical constant $c_2>0$. This concludes the proof of the theorem.

\section{Proof of Theorem \ref{thm:lowerboundinf}}\label{prf:lowerboundinf}

For an estimator $\widehat{\pi}=\widehat{\pi}_1,\ldots, \widehat{\pi}_p$, we \new{recall the definition of $\err$:}
$$\err(\widehat{\pi},\pi^\star)=\min_{\psi\in \cS_n}\frac{1}{np}\sum_{u=1}^n\sum_{j=1}^p \1\ac{\psi(\widehat{\pi}_j(u))\neq \pi^{\star}_j(u)}\, .$$
In order to lower-bound the optimal error $\err_{\mathrm{opt}}:=\inf_{\widehat{\pi}:\mathbb{R}^{E\times p}\to \pa{\cS_n}^p} \E\cro{\err(\widehat{\pi}, \pi^{\star})}$, we shall proceed in two steps;\begin{itemize}
    \item Step 1: we lower-bound $\err_{\mathrm{opt}}$ with respect to the optimal error of estimating $\pi^{\star}_p$ \new{given} all the other permutations,
    \item Step 2: we lower-bound the optimal error of estimating $\pi^{\star}_p$ using Fano's lemma, which requires to upper-bound the Kullback–Leibler divergence between the joint law conditionally on $\pi^{\star}_p=\id$ and the law conditionally on $\pi^\star_p=\pi$ for any $\pi\in \cS_n$.
\end{itemize}
\proofstep{Step 1.}
Let us now lower-bound the $\err_{\mathrm{opt}}$ with respect to the optimal error of estimating $\pi^{\star}_p$. For $j\neq j'\in [p]$ \new{and a global estimator $\widehat{\pi}$, we define the error of $\widehat{\pi}$ between $G_j$ and $G_{j'}$} as $$\err_{j,j'}(\widehat{\pi}, \pi^\star):=\min_{\psi\in \cS_n}\frac{1}{2n}\sum_{u=1}^n \1\ac{\psi(\widehat{\pi}_j(u))\neq \pi^{\star}_j(u)}+\1\ac{\psi(\widehat{\pi}_{j'}(u))\neq \pi^\star_{j'}(u)}\, .$$
We have, for any $\widehat{\pi}$ and a minimizer $\psi$ in the definition of $\err(\widehat{\pi}, \pi^{\star})$,
\begin{align*}
    \err(\widehat{\pi}, \pi^{\star})= &\frac{1}{np}\sum_{u=1}^n\sum_{j=1}^p \1\ac{\psi(\widehat{\pi}_j(u))\neq \pi^{\star}_j(u)}\\
    =& \frac{1}{np (2p-2)}\sum_{\new{ 1 \leq j\neq j' \leq p}} \sum_{u=1}^n \1\ac{\psi(\widehat{\pi}_j(u))\neq \pi^{\star}_j(u)}+\1\ac{\psi(\widehat{\pi}_{j'}(u))\neq \pi^{\star}_{j'}(u)}\\
    \geq& \frac{1}{p(p-1)}\sum_{\new{ 1 \leq j\neq j' \leq p}}  \err_{j,j'}\pa{\widehat{\pi}, \pi^\star},
\end{align*}
which leads to, taking to the expectation,
\begin{equation*}
   \E\cro{\err(\widehat{\pi},\pi^\star)}\geq  \frac{1}{p(p-1)}\sum_{\new{ 1 \leq j\neq j' \leq p}}  \E\cro{\err_{j,j'}\pa{\widehat{\pi}, \pi^\star}}\, .
\end{equation*}
\proofstep{Step 2.} Fix $(j,j')=(p-1, p)$ by symmetry and let us find a lower bound for $\inf_{\widehat{\pi}:\mathbb{R}^{E\times p}\to \pa{\cS_n}^p} \E\cro{\err_{p-1,p}\pa{\widehat{\pi}, \pi^\star}}$. Remark that $\err_{p-1,p}\pa{\widehat{\pi}, \pi^\star}$ only depends on the two last permutations $\widehat{\pi}_{p-1}, \widehat{\pi}_p$ so we will lower-bound $$\err_{p-1,p}^{\mathrm{opt}} :=\inf_{\widehat{\pi}_{p-1}, \widehat{\pi}_p \in \cS_n}\E\cro{\err\pa{\pa{\widehat{\pi}_{p-1}, \widehat{\pi}_p}, \pa{\pi^{\star }_{p-1}, \pi^{\star}_p}}}\, .$$
Let $\pa{\widehat{\pi}^{\mathrm{opt}}_{p-1}, \widehat{\pi}^{\mathrm{opt}}_{p}}$ minimizing the above quantity. For $\pi\in \cS_n$, let $\P_\pi:=\P_{\id,\ldots,\id,\pi}$ the law of $G_1, \ldots, G_p$ conditionally on $\pi^{\star}_1=\id, \ldots, \pi^{\star}_{p-1}=\id, \pi^{\star}_p=\pi$. We recall that we denote $\P$ the law of $G_1, \ldots, G_p$ with $\pi^\star_1,\ldots, \pi^{\star}_p$ taken independently and uniformly on $\cS_n$. \new{Consider} the following procedure for recovering $\pi$ under $\P_{\pi}$.
\begin{enumerate}
    \item draw $\sigma_1, \ldots , \sigma_p$ uniformly and independently over $\cS_n$,
    \item relabel $G_1, \ldots, G_p$ accordingly to $\sigma_1, \ldots , \sigma_p$, and denote the relabeled graphs $\widetilde{G}_1, \ldots, \widetilde{G}_p$
    \item compute $\pa{\widehat{\pi}^{\mathrm{opt}}_{p-1}, \widehat{\pi}^{\mathrm{opt}}_{p}}(\widetilde{G}_1, \ldots, \widetilde{G}_p)$;
    \item Take $\widehat{\pi}= \pa{\widehat{\pi}^{\mathrm{opt}}_{p-1}}^{-1}\circ  \sigma_{p-1} \circ \pa{\sigma_{p}}^{-1}\circ \widehat{\pi}^{\mathrm{opt}}_{p} \in \cS_n$ \new{as an estimator for $\hat\pi$}.
\end{enumerate}
We remark that, under the law $\mathbb{P}_{\pi}$ for $G_1,\ldots, G_p$, the permuted graphs $\widetilde{G}_1, \ldots, \widetilde{G}_p$ have joint law $\P$. Thus, $$\E_{\pi}\cro{\err\pa{\pa{\widehat{\pi}^{\mathrm{opt}}_{p-1}, \widehat{\pi}^{\mathrm{opt}}_p}\pa{\widetilde{G}_1, \ldots, \widetilde{G}_p}, \pa{\sigma_{p-1}, \sigma_p\circ \pi}}}=\err^{\mathrm{opt}}_{p-1,p}\, .$$
In the following, we denote $\delta_H(\pi,\pi'):=\frac{1}{n}\sum_{u=1}^n\1\ac{\pi(u)\neq \pi'(u)}$ the Hamming distance between two permutations. We have 
\begin{align*}
\err\pa{\pa{\widehat{\pi}^{\mathrm{opt}}_{p-1}, \widehat{\pi}^{\mathrm{opt}}_p}, \pa{\sigma_{p-1}, \sigma_p\circ \pi}}=& \err\pa{\pa{\sigma_p^{-1}\circ \widehat{\pi}^{\mathrm{opt}}_{p-1}, \sigma_{p-1}^{\mathrm{opt}}\circ \widehat{\pi}_p^{\mathrm{opt}}}, \pa{\id, \pi}}\\
=&\frac{1}{2}\min_{\psi\in \cS_n} \cro{\delta_H\pa{\psi\circ \sigma_{p-1}^{-1}\circ \widehat{\pi}^{\mathrm{opt}}_{p-1}, \id}+\delta_H\pa{\psi\circ \sigma_p^{-1}\circ \widehat{\pi}^{\mathrm{opt}}_p, \pi}} \, .
\end{align*}
Let us now fix $\psi\in \cS_n$. Then, using the triangular equality, 
\begin{align*}
    \delta_H\pa{\widehat{\pi}, \pi}=&\delta_H\pa{\pa{\widehat{\pi}_{p-1}^{\mathrm{opt}}}^{-1}\circ \sigma_{p-1}\circ \pa{\sigma_p}^{-1}\circ \widehat{\pi}^{\mathrm{opt}}_{p}, \pi}\\
    \leq& \delta_{H}\pa{\psi\circ \sigma_{p}^{-1}\circ \widehat{\pi}^{\mathrm{opt}}_p, \pi}+\delta_{H}\pa{\psi\circ \sigma_{p}^{-1}\circ \widehat{\pi}^{\mathrm{opt}}_p, \pa{\widehat{\pi}_{p-1}^{\mathrm{opt}}}^{-1}\circ \sigma_{p-1}\circ \pa{\sigma_p}^{-1}\circ \widehat{\pi}^{\mathrm{opt}}_{p}}\\
    =& \delta_H\pa{\psi\circ \sigma_{p}^{-1}\circ \widehat{\pi}^{\mathrm{opt}}_p, \pi}+ \delta_{H}\pa{\psi\circ \sigma_p^{-1}\circ \widehat{\pi}^{\mathrm{opt}}_p},
\end{align*}
and, taking the infimum over all $\psi\in \cS_n$ and taking to the expectation, we get 
$$\E_{\pi}\cro{\delta_H\pa{\widehat{\pi}, \pi}}\leq 2\err_{p-1,p}^{\mathrm{opt}}\leq  2\err_{\mathrm{opt}}\, .$$
We conclude the proof of the theorem with the next lemma, whose proof is postponed to Section \ref{prf:minimaxrisk}.

\begin{lemma}\label{lem:minimaxrisk}
Suppose $n\geq n_0$ with $n_0$ some numerical constant. There exists numerical constants $c,c'$ such that, when $\rho\leq c\pa{\frac{\log(n)}{n}\vee \sqrt{\frac{\log(n)}{np}}}$,  

$$\inf_{\widehat{\pi}:\mathbb{R}^{E\times p}\to \cS_n}\sup_{\pi\in \cS_n}\E_{\pi}\cro{\delta_H(\widehat{\pi}, \pi)}\geq c\, .$$
\end{lemma}

\subsection{Proof of Lemma \ref{lem:minimaxrisk}}\label{prf:minimaxrisk}

We recall that, under $\P_{\pi}$, we observe $G_1,\ldots, G_p$ conditionally over $\pi^{\star}_j=\id$ for $j\in [p-1]$ and over $\pi^{\star}_p=\pi$. We shall lower-bound $$\inf_{\widehat{\pi}:\mathbb{R}^{E\times p}\to \cS_n}\sup_{\pi\in \cS_n}\E_{\pi}\cro{\delta_H(\widehat{\pi},\pi)}\, .$$
To do so, we shall use an adaptation of Fano's lemma, stated as Lemma \ref{lem:Fanoerr} (see e.g. Section 3.2.2 of \cite{HDS2} for a proof). 
\begin{lemma}\label{lem:Fanoerr}
For any finite set $A\subseteq \cS_n$, we have the following inequality $$\inf_{\widehat{\pi}}\frac{2}{|A|}\sum_{\pi\in A}\E_{\pi}\cro{\delta_H(\widehat{\pi}, \pi)}\geq (\min_{\pi\neq \pi'\in A}\delta_H(\pi,\pi')) \times \pa{1-\frac{1+\frac{1}{|A|}\sum_{\pi\in A}\KL\pa{\P_{\pi},\P_{\id}}}{\log(|A|)}}\, .$$
\end{lemma}
In the light of this lemma, we shall find a set of permutations which are all at great Hamming distance one from the others, but such that $\KL\pa{\P_{\pi},\P_{\id}}$ remains small. Such a set is given Lemma 16 of \cite{Collier16}, which we state here.
\begin{lemma}\label{lem:numberpermutations}
For any interger $n\geq 4$, there exists a set $A\subset \cS_n$ such that
\begin{enumerate}
    \item $|A|\geq \pa{\frac{n}{24}}^{n/6}$ ;
    \item For every $\pi\neq \pi'\in A$, $\delta_H(\pi,\pi')\geq 3/8$.
\end{enumerate}
\end{lemma}
It remains to compute the Kullback-Leibler divergence $\KL\pa{\P_{\pi},\P_{\id}}$ for any $\pi\in A$. We refer to Section \ref{prf:KLcomputation} for a proof of the next lemma. 
\begin{lemma}\label{lem:KLcomputation}
\new{Let $A$ be a subset of $\cS_n$ satisfying conditions of \cref{lem:numberpermutations}}. For any $\pi\in A$, $$\KL\pa{\P_{\pi},\P_{\id}}\leq \frac{n(n-1)\rho^2 (p-1)}{(1-\rho)(1+\rho(p-1))} \, .$$
\end{lemma}
Combining Lemmas \ref{lem:KLcomputation}, \ref{lem:numberpermutations} and \ref{lem:Fanoerr}, we get, whenever $n\geq 4$,

\begin{equation*}
    \inf_{\widehat{\pi}:\mathbb{R}^{E\times p}\to \cS_n}\sup_{\pi\in A}\E_{\pi}\cro{\delta_H\pa{\widehat{\pi}, \pi}}\geq \frac{3}{8}\pa{1-\frac{1+\frac{n(n-1)\rho^2 (p-1)}{(1-\rho)(1+\rho(p-1))}}{\frac{n}{6}\log(\frac{n}{24})}}\, .
\end{equation*}

First, assume $\rho \leq c\sqrt{\frac{\log(n)}{np}}$, for a small enough numerical constant $c$. Then,

\begin{align*}
    \inf_{\widehat{\pi}:\mathbb{R}^{E\times p}\to \cS_n}\sup_{\pi\in A}\E_{\pi}\cro{\delta_H\pa{\widehat{\pi}, \pi}}\geq  \frac{3}{8}\pa{1-\frac{6+12n^2\rho^2 p}{n\log(n/24)}} 
    \geq  \frac{3}{16},
\end{align*}
provided the constant $c$ is small enough and $n$ large enough. 

Second, assume $\rho\leq c\frac{\log(n)}{n}$ with $c$ a small enough numerical constant. This gives
\begin{align*}
    \inf_{\widehat{\pi}:\mathbb{R}^{E\times p}\to \cS_n}\sup_{\pi\in A}\E_{\pi}\cro{\delta_H\pa{\widehat{\pi}, \pi}} \geq \frac{3}{8}\pa{1-\frac{6+\frac{12n^2\rho^2 (p-1)}{\rho(p-1)}}{n\log(n/24)}} \geq \frac{3}{8}\pa{1-\frac{6+12n^2\rho}{n\log(n/24)}} \geq \frac{3}{16},
\end{align*}
provided $n$ is large enough and $c$ small enough. This concludes the proof of the lemma. 

\subsubsection{Proof of Lemma \ref{lem:KLcomputation}}\label{prf:KLcomputation}

Let us fix $\pi\in A$ and let us compute $\KL(\P_{\pi}, \P_{\id})=\E_\pi\cro{\log\pa{\frac{d\P_{\pi}}{d\P_{\id}}}}$. We recall that the law $\P_\pi$ (resp. $\P_{\id}$) can be generated as follows.\begin{enumerate}
    \item draw $\pa{G^\star_e}_{e\in E}$ independently and $\mathcal{N}\pa{0,1}$;
    \item Independently from $\pa{G^\star_e}_{e\in E}$, draw $\pa{Z_{je}}_{e\in E, j\in [p]}$ i.i.d. $\mathcal{N}\pa{0,1}$;
    \item For $j\leq p-1$, take $G_{je}=\sqrt{\rho} G^\star_e+\sqrt{1-\rho}Z_{je}$;
    \item \new{For $j=p$, take $G_{pe}=\sqrt{\rho} G^\star_{\Pi(e)}+\sqrt{1-\rho}Z_{pe}$ (resp. $G_{pe}=\sqrt{\rho} G^\star_{e}+\sqrt{1-\rho}Z_{pe}$)}, where $\Pi(e)=\ac{\pi(u), \pi(u')}$ for $e=\ac{u,u'}$.
\end{enumerate}
In order to compute $\KL(\P_\pi, \P_{\id})$, we first compute the likelihood ratio $\frac{d\P_{\pi}}{d\P_{\id}}$. We have, since the $G^\star_e$'s are drawn independently, 

\begin{align*}
    \frac{d\P_{\pi}}{d\P_{\id}}=&\prod_{e\in E}\frac{\E_{G^{\star}}\cro{\exp\pa{\frac{-1}{2(1-\rho)}\cro{\sum_{j=1}^{p-1}\pa{G_{je}-\sqrt{\rho}G_e^{\star}}^2+\pa{G_{p\Pi^{-1}(e)}-\sqrt{\rho}G_e^{\star}}^2}}}}{\E_{G^{\star}}\cro{\exp\pa{\frac{-1}{2(1-\rho)}\sum_{j=1}^{p}\pa{G_{je}-\sqrt{\rho}G_e^{\star}}^2}}}\\
    =&\prod_{e\in E}\frac{\E_{G^{\star}}\cro{\exp\pa{\frac{-p}{2(1-\rho)}\cro{(\sqrt{\rho}G_e^{\star})^2-2\sqrt{\rho}G_e^{\star}\overline{G_{\Pi^{-1}(e)}}}}}}{\E_{G^{\star}}\cro{\exp\pa{\frac{-p}{2(1-\rho)}\cro{(\sqrt{\rho}G_e^{\star})^2-2\sqrt{\rho}G_e^{\star}\overline{G_{e}}}}}},
\end{align*}
where we denote $\overline{G_{\Pi^{-1}(e)}}=\frac{1}{p}\pa{\sum_{j=1}^{p-1}G_{je}+G_{p\Pi^{-1}(e)}}$ and $\overline{G}_e=\frac{1}{p}\sum_{j=1}^pG_{je}$. Let us compute the numerator of the above equality. Recall that $G^{\star}_e\sim \mathcal{N}\pa{0,1}$, so that 
\begin{align*}
    \E_{G^{\star}}\cro{\exp\pa{\frac{-p}{2(1-\rho)}\cro{(\sqrt{\rho}G_e^{\star})^2-2\sqrt{\rho}G_e^{\star}\overline{G_{\Pi^-1(e)}}}}}=&\frac{1}{\sqrt{2\pi}}\int \exp\pa{\frac{-x^2}{2}}\exp\pa{\frac{-p}{2(1-\rho)}\cro{\rho x^2-2\sqrt{\rho}x\overline{G_{\Pi^-1(e)}}}}dx\\
    =&\frac{1}{\sqrt{2\pi}}\int\exp\pa{\frac{-1}{2}\cro{\pa{1+\frac{\rho p}{1-\rho}}x^2-2p\frac{\sqrt{\rho}\overline{G_{\Pi^-1(e)}}}{1-\rho}x}}dx\\
    =&\frac{1}{\sqrt{2\pi}}\int \exp\pa{\frac{-(1+\rho(p-1))}{2(1-\rho)}\cro{x^2-2p\frac{\sqrt{\rho}\overline{G_{\Pi^-1(e)}}}{1+\rho(p-1)}x}}dx\\
    =&\frac{1}{\sqrt{2\pi}}\int\exp\pa{\frac{-(1+\rho(p-1))}{2(1-\rho)}\cro{\pa{x-p\frac{\sqrt{\rho}\overline{G_{\Pi^-1(e)}}}{1+\rho(p-1)}}^2}}dx\\
    &\times\exp\pa{\frac{(1+\rho(p-1))}{2(1-\rho)}\frac{p^2\rho \overline{G_{\Pi^-1(e)}}^2}{\pa{1+\rho(p-1)}^2}}dx\\
    =&\sqrt{\frac{1+\rho(p-1)}{1-\rho}}\exp\pa{\frac{p^2\rho\overline{G_{\Pi^-1(e)}}^2}{2(1-\rho)(1+\rho(p-1))}},
\end{align*}
and, similarly, we have

\begin{equation*}
    \E_{G^{\star}}\cro{\exp\pa{\frac{-p}{2(1-\rho)}\cro{(\sqrt{\rho}G_e^{\star})^2-2\sqrt{\rho}G_e^{\star}\overline{G_{e}}}}}=\sqrt{\frac{1+\rho(p-1)}{1-\rho}}\exp\pa{\frac{p^2\rho\overline{G_{e}}^2}{2(1-\rho)(1+\rho(p-1))}},
\end{equation*}
so, \new{wrapping things up,} 
\begin{equation}\label{eq:computationKLratio}
    \frac{d\P_{\pi}}{d\P_{\id}}=\prod_{e\in E}\exp\pa{\frac{p^2\rho\pa{\overline{G_{\Pi^-1(e)}}^2-\overline{G_{e}}^2}}{2(1-\rho)(1+\rho(p-1))}}\, .
\end{equation}
Plugging this equality in the definition of the KL divergence leads to 
\begin{equation}\label{eq:computationKLratio2}
    \KL\pa{\P_\pi, \P_{\id}}=\frac{p^2\rho}{2(1-\rho)(1+\rho(p-1))}\sum_{e\in E}\E_\pi\cro{\overline{G_{\Pi^-1(e)}}^2-\overline{G_{e}}^2}\, .
\end{equation}
Fix $e\in E$. \new{We have}
$$\E_\pi\cro{\overline{G_{\Pi^-1(e)}}^2}=\frac{1}{p^2}\E_\pi\cro{\pa{\sum_{j\in [p-1]}G_{je}+G_{p\Pi^{-1}(e)}}^2} =\frac{1}{p^2}\pa{p+\rho p(p-1)}, $$
\new{and on the other hand, }
\begin{equation*}
    \E_\pi\cro{\overline{G_{e}}^2}\geq \frac{1}{p^2}\pa{p+\rho(p-1)(p-2)}\, .
\end{equation*}
Hence,
\begin{equation*}
    \KL\pa{\P_\pi, \P_{\id}}\leq \frac{n(n-1)\rho^2 (p-1)}{(1-\rho)(1+\rho(p-1))},
\end{equation*}
which concludes the proof of the lemma.

\section{Proof of Theorem \ref{thm:lowerboundinfperfect}}\label{lowerboundinfperfect}

In this section, we seek to lower-bound $$\inf_{\widehat{\pi}:\mathbb{R}^{E\times p}\to \pa{\cS_n}^p}\P\cro{\err(\widehat{\pi}, \pi^{\star})\neq 0}\, . $$

We will rely on Fano's lemma to do so. For that purpose, we need to consider the law of $G_1,\ldots,G_p$ conditionally on $\pi^{\star}$. For any $\pi\in (\cS_n)^p$, we \new{now write} $\P_\pi$ the joint law of $G_1,\ldots, G_p$ conditionally over $\pi^{\star}=\pi$. We write $\P$ the joint law of $G_1,\ldots, G_p$ for $\pi^{\star}$ taken randomly. 

\begin{lemma}\label{lem:reductionexact}
$$\inf_{\widehat{\pi}:\mathbb{R}^{E\times p}\to \pa{\cS_n}^p}\P\cro{\err(\widehat{\pi}, \pi^{\star})\neq 0}= \inf_{\widehat{\pi}:\mathbb{R}^{E\times p}\to \pa{\cS_n}^p}\sup_{\pi\in \pa{\cS_n}^p}\P_\pi\cro{\err(\widehat{\pi}, \pi^{\star})\neq 0}$$
\end{lemma}

\begin{proof}[Proof of Lemma \ref{lem:reductionexact}]
    Let $\widehat{\pi}:\mathbb{R}^{E\times p}\to \pa{\cS_n}^p$. \new{Let $\pi\in (\cS_n)^p$ and let $\sigma = (\sigma_1,\ldots, \sigma_p) \in (\cS_n)^p$ be i.i.d. uniformly drawn independently from everything else. Define $H_{je} :=G_{j\sigma_j(e)}$. Then, for $G = (G_1,\ldots, G_p) \sim \P_\pi$, we have $H := (H_1,\ldots, H_p) \sim \P$}. Thus, 
    \begin{equation*}
        \P\cro{\err(\widehat{\pi}(G), \pi^{\star})\neq 0}= \P_\pi\cro{\err\pa{\widehat{\pi}\pa{H}, \sigma\circ \pi}\neq 0}=\P_{\pi}\cro{\err\pa{\sigma^{-1}\circ \widehat{\pi}\pa{H}, \pi}\neq 0},
    \end{equation*}
    \new{and this for all $\pi\in (\cS_n)^p$ which in turn implies} $$\P\cro{\err(\widehat{\pi}, \pi^{\star})\neq 0}= \sup_{\pi\in (\cS_n)^p}\P_{\pi}\cro{\err\pa{\sigma^{-1}\circ \widehat{\pi}\pa{H}, \pi}\neq 0}\, .$$
    We conclude with the sought equality
    $$\inf_{\widehat{\pi}:\mathbb{R}^{E\times p}\to \pa{\cS_n}^p}\P\cro{\err(\widehat{\pi}, \pi^{\star})\neq 0}= \inf_{\widehat{\pi}:\mathbb{R}^{E\times p}\to \pa{\cS_n}^p}\sup_{\pi\in \pa{\cS_n}^p}\P_\pi\cro{\err(\widehat{\pi}, \pi^{\star})\neq 0}$$
\end{proof}

In the light of Lemma \ref{lem:reductionexact}, it is sufficient to lower-bound $\inf_{\widehat{\pi}:\mathbb{R}^{E\times p}\to \pa{\cS_n}^p}\sup_{\pi\in \pa{\cS_n}^p}\P_\pi\cro{\err(\widehat{\pi}, \pi^{\star})\neq 0}$. We shall appeal to Fano's Lemma that we recall here (see e.g p. 57 of \cite{HDS2}).

\begin{lemma}\label{lem:Fanoexact}
    Let $(\P_s)_{s\in [\ell]}$ a set of probability distributions on the same space $\mathcal{Y}$. \new{Let $\mathbb{Q}$ a probability distribution on $\mathcal{Y}$ such that $\mathbb{P}_s \ll \mathbb{Q}$ for all $s\in [\ell]$}. Then, $$\min_{\widehat{J}:\mathcal{Y}\to [\ell]}\frac{1}{\ell}\sum_{s\in [\ell]}\mathbb{P}_s\cro{\widehat{J}\neq s}\geq 1-\frac{1+\frac{1}{\ell}\sum_{s\in [\ell]}\KL\pa{\mathbb{P}_s, \mathbb{Q}}}{\log(\ell)}\, .$$
\end{lemma}

Let us consider $\mathbb{Q}=\mathbb{P}_{\id}$ for applying Lemma \ref{lem:Fanoexact}. Our goal is to find a large set of $p$-tuples of permutations such that $\KL\pa{\P_\pi, \P_{\id}}$ remains small. To do so, for $j\in [p]$ and $e=\ac{u,u'}\in E$, we define $\pi^{(j, e)}$ the $p$-tuple of permutations by $\pi^{(j, e)}_{j'}=\id$ whenever $j\neq j'$ and  $\pi^{(j, e)}_j$ is the transposition swapping $u$ and $u'$. 

\begin{lemma}\label{lem:KLpartial}
    For all $j\in [p]$ and $e\in E$, we have $$\KL\pa{\P_{\pi^{(j, e)}}, \P_{\id}}=\frac{\pa{2n-3}\rho^2(p-1)}{(1-\rho)\pa{1+\rho(p-1)}} \, .$$
\end{lemma}

Combining together Lemma \ref{lem:KLpartial} and Lemma \ref{lem:Fanoexact}, we deduce that $$\min_{\widehat{J}:\mathbb{R}^{E\times p}\to [p]\times E}\sum_{j, e\in [p]\times E}\P_{\pi^{(j, e)}}\cro{\widehat{J}\neq (j, e)}\geq 1-\frac{1+\frac{\pa{2n-3}\rho^2(p-1)}{(1-\rho)\pa{1+\rho(p-1)}}}{\log(pn(n-1)/2)},$$
and thus, 

$$\inf_{\widehat{\pi}:\mathbb{R}^{E\times p}\to \pa{\cS_n}^p}\sup_{\pi\in \pa{\cS_n}^p}\P_\pi\cro{\err(\widehat{\pi}, \pi^{\star})\neq 0}\geq 1-\frac{1+\frac{\pa{2n-3}\rho^2(p-1)}{(1-\rho)\pa{1+\rho(p-1)}}}{\log(pn(n-1)/2)}\, .$$

First, \new{assume} that $\rho \leq c\sqrt{\frac{\log(np)}{np}}$ for a small enough numerical constant $c>0$. Then, \begin{align*}
\inf_{\widehat{\pi}:\mathbb{R}^{E\times p}\to \pa{\cS_n}^p}\sup_{\pi\in \pa{\cS_n}^p}\P_\pi\cro{\err(\widehat{\pi}, \pi^{\star})\neq 0}\geq 1-\frac{1+8n\rho^2p}{\log(pn(n-1)/2)}
\geq 1-\eps,
\end{align*}
provided $c$ is small enough and $n$ is large enough. 

\new{Second, assume that} $\rho\leq c\frac{\log(np)}{n}$. Then,  
\begin{align*}
\inf_{\widehat{\pi}:\mathbb{R}^{E\times p}\to \pa{\cS_n}^p}\sup_{\pi\in \pa{\cS_n}^p}\P_\pi\cro{\err(\widehat{\pi}, \pi^{\star})\neq 0}\geq 1-\frac{1+8n\rho}{\log(np(n-1))}
\geq 1-\eps,
\end{align*}
provided $c$ is small enough and $n$ is large enough. This concludes the proof of the theorem.

\subsection{Proof of Lemma \ref{lem:KLpartial}}\label{prf:KLpartial}
\new{By symmetry, it is enough to prove the result for $j=p$ and $e = \ac{1,2}$. We denote $\tau$ the transposition swapping $1$ and $2$. Starting back from \eqref{eq:computationKLratio2}}, we have 
$$\KL\pa{\P_{\pi}, \P_{\id}}=\frac{p^2\rho}{2(1-\rho)(1+\rho(p-1))}\sum_{e\in E}\E_{\pi}\cro{\overline{G_{\Pi^-1(e)}}^2-\overline{G_{e}}^2},$$ where we recall that $\overline{G_{e}} =\frac{1}{p}\sum_{j\in [p]}G_{je}$ and $\overline{G_{\Pi^-1(e)}}=\frac{1}{p}\pa{\sum_{j\in [p-1]}G_{je}+G_{p \tau(e)}}$. \new{In the above sum, all the terms such that $e \cap \ac{u,u'}=\varnothing$ make a null contribution}. For the other terms \new{-- there are $\frac{n(n-1)}{2} - \frac{(n-2)(n-3)}{2} = 2n-3$ such terms --}  we use $$\E_{\pi}\cro{\overline{G_{\Pi^-1(e)}}^2}=\frac{1}{p^2}\cro{p+\rho p(p-1)} \quad \mbox{and} \quad \E_\pi\cro{\overline{G_e}^2}=\frac{1}{p^2}\pa{p+\rho(p-1)(p-2)}\, .$$
We can conclude $$\KL\pa{\P_{\pi}, \P_{\id}}=\frac{\pa{2n-3}\rho^2(p-1)}{(1-\rho)\pa{1+\rho(p-1)}}\, .$$

\section{Proof of Theorem \ref{thm:lowdegreealignment}}\label{prf:lowdegreealignment}

The minimization problem \eqref{eq:MMSE} defining $\MMSE_{\leq D}$ is separable, and since the random variables $M^{\star}_{(u,j), (u', j')}$ are exchangeable, the $\MMSE_{\leq D}$ can be reduced to $$\MMSE_{\leq D}:=\inf_{f\in \mathbb{R}_D[G]}\E\cro{(f(G)-x)^2},$$
where $x=\1\ac{\pi^{\star}_1(1)=\pi^{\star}_2(1)}$. In order to lower-bound $\MMSE_{\leq D}$, we proceed in two steps;\begin{enumerate}
    \item We consider the case where $p=2$ and we use the technics of \cite{WeinSchramm} and \cite{even2025a};
    \item We reduce the general problem of aligning $p$ Gaussian graphs to the problem of aligning one graph with some signal graph $G^\star$ and we use the lower bound obtained for two graphs.
\end{enumerate}
Interestingly, the method of \cite{WeinSchramm} is not sufficient for directly considering the general case with $p$ Gaussian graphs. A technical analysis based directly on \cite{WeinSchramm}, using the technics from \cite{even2025b} for the weak dependencies, would only get the sought result for $p\leq n$. The reduction step to the problem of two graphs has therefore two benefits; (i) simplifying the computations (ii) allowing us to deal with the regime $p\geq n$.

\subsection{Case $p=2$.}

We recall that the graphs $G_1, G_2$ can be sampled from a signal graph $G^\star$ with $\pa{G^\star_e}_{e\in E}$ taken independently and $\mathcal{N}\pa{0,1}$, and from noise graphs $Z_1, Z_2$ with also independent and $\mathcal{N}\pa{0,1}$ edges. With those graphs, we construct $G_{je}= \sqrt{\rho}G^\star_{\Pi_j^\star(e)}+\sqrt{1-\rho}Z_{je}$.

For $j\in \ac{1,2}$ and $e\in E$, we denote $Y_{je}=\frac{\sqrt{\rho}}{\sqrt{1-\rho}}G^{\star}_{p\Pi^{\star}(e)}+Z_{je}=\frac{1}{\sqrt{1-\rho}}G_{je}$. The set of estimators which are polynomials of degree at most $D$ of $G_1, G_2$ is equal to the set of polynomials of degree at most $D$ of $Y_1, Y_2$. Hence, $$\MMSE_{\leq D}=\inf_{g\in \mathbb{R}_{D}\cro{Y_1, Y_2}}\E\cro{\pa{g(Y_1, Y_2)-x}^2}\, .$$

The model considered is a particular instance of the Additive Gaussian Noise model considered in \cite{WeinSchramm}. Therefore, we can use Theorem 2.2 from \cite{WeinSchramm} that we transcript here with our notation. Before that, let us introduce the notion of \emph{joint cumulants} (see e.g \cite{novak2014three} for more details).

\paragraph{Backround on cumulants.} Given $Y_1,\ldots, Y_l$ random variables on the same probability space $\mathcal{Y}$, we define their \emph{joint cumulant} as the quantity $$\cumul\pa{Y_1,\ldots, Y_l}:=\sum_{G\in \mathcal{P}([\ell])}m(G)\prod_{R\in G}\E\cro{\prod_{s\in [\ell]}Y_s},$$
with $\mathcal{P}([\ell])$ the set of all partitions of $[\ell]$ and, for $G\in \mathcal{P}([\ell])$, $m(G)=\pa{-1}^{|G|-1}\pa{|G|-1}!$. The \emph{joint cumulants} of $Y_1,\ldots, Y_l$ can also be computed recursively with the formula

\begin{equation}\label{eq:recursioncumulantgeneral}
    \cumul\pa{Y_1,\ldots, Y_l}=\E\cro{Y_1\ldots Y_l}-\sum_{A\subsetneq [2,l]}\cumul\pa{Y_1, \pa{Y_s}_{s\in A}}\E\cro{\prod_{s\in [2,l]\setminus A}Y_s}\, .
\end{equation}

\begin{prop}\label{thm:schrammwein}\cite{WeinSchramm}
    $$\MMSE_{\leq D}\geq \E\cro{x^2}-\underset{|\alpha|\leq D}{\sum_{\alpha=\pa{\alpha_1, \alpha_2}\in (\N^{E})^2}}\pa{\frac{\rho}{1-\rho}}^{|\alpha|}\frac{\kappa_{x,\alpha}^2}{\alpha!},$$ where, for $\alpha=\pa{\alpha_1, \alpha_2}\in (\N^{E})^2$, $\alpha!=\prod_{e\in E}\alpha_{1e}! \alpha_{2e}!$, $|\alpha|=\sum_{e\in E}\alpha_{1e}+\alpha_{2e}$, and, $$\kappa_{x,\alpha}=\cumul\pa{x, G^\star_{\alpha_1}, G^\star_{\alpha_2}},$$ with $G_{\alpha_{j}}$ the multiset containing $\alpha_{je}$ copies of $G_{je}$ for all $e\in E$. 
\end{prop}

In the light of Proposition \ref{thm:schrammwein}, \new{it is sufficient to upper-bound the cumulant $\kappa_{x,\alpha}$ for all $\alpha$ with $|\alpha|\leq D$ in order to lower-bound $\MMSE_{\leq D}$}. Such an upper-bound is provided by the next lemma, whose proof is postponed to Section \ref{lem:controlcumulant}. For $\alpha=\pa{\alpha_1, \alpha_2}\in (\N^{E})^2$ and $j\in \ac{1,2}$, we shall define $$\supp(\alpha_{j})=\ac{u\in [n], \pa{\alpha_{j}}_{u:}\neq 0\quad or \quad \pa{\alpha_{j}}_{:u}\neq 0},$$ which is the number of nodes with positive degree when seeing $\alpha_j$ as a multigraph with vertex set $[n]$.

\begin{lemma}\label{lem:controlcumulant}
    For all $\alpha=\pa{\alpha_1, \alpha_2}\in (\N^{E})^2$, we have $$|\kappa_{x,\alpha}|\leq \1\ac{|\alpha_1|=|\alpha_2|}\pa{|\alpha|(1+|\alpha|/2)}^{|\alpha|/2}\pa{\frac{\sqrt{2}}{\sqrt{n-1-|\alpha|}}}^{|\ac{1}\cup \supp(\alpha_1)|+|\ac{1}\cup \supp(\alpha_2)|},$$
\end{lemma}
\new{Note that} when $\pa{\alpha_1,\alpha_2}\neq (0,0)$ with $|\alpha_1|=|\alpha_2|$, for $j=1,2$, we have,  $2\leq |\ac{1}\cup \supp(\alpha_j)|\leq |\alpha|+1$.

\begin{lemma}\label{lem:combinatoricsalpha}
    Let $d\in [D]$ and $2\leq m_1,m_2\leq d+1$. There are at most $n^{m_1+m_2-2}d^{2d}$ elements $\alpha$ satisfying $|\alpha|=d$, $|\supp(\alpha_1)\cup \ac{1}|=m_1$ and $|\supp(\alpha_2)\cup \ac{1}|=m_2$.
\end{lemma}

\begin{proof}[Proof of Lemma \ref{lem:combinatoricsalpha}]
There are at most $n^{m_1+m_2-2}$ possibilities for choosing two sets $S_1$ and $S_2$ both containing $1$ and of size $m_1$ and $m_2$. Let us then count the number of multi-graph $\alpha_1$ with $d/2$ edges and with $\supp(\alpha_1)\cup \ac{1}=S_1$;
\begin{itemize}
    \item if $m_1\leq d$, we choose for each edge two extremities in $S_1$. Since we have to choose $d/2$ edges, that makes at most $d^{d}$ possibilities,
    \item if $m_1= d+1$, we choose for each edge two extremities in $S_1\setminus \ac{1}$. Since we have to choose $d/2$ edges, that also makes at most $d^{d}$ possibilities.
\end{itemize}

The same thing holds for the the number of multi-graph $\alpha_2$ with $d/2$ edges and with $\supp(\alpha_2)\cup \ac{1}=S_2$. In total, we have at most $d^{2d}$ possibilities for choosing a multigraph $\alpha=(\alpha_1,\alpha_2)$, given $\supp(\alpha_1)\cup \ac{1}=S_1$ and $\supp(\alpha_2)\cup \ac{1}=S_2$.  
\end{proof}

Combining Lemma \ref{lem:controlcumulant}, Lemma \ref{lem:combinatoricsalpha} and Proposition \ref{thm:schrammwein}, \new{we get}

\begin{align*}
    \underset{|\alpha|\leq D}{\sum_{\alpha=\pa{\alpha_1, \alpha_2}\in (\N^{E})^2}}\frac{\rho}{1-\rho}\frac{\kappa_{x,\alpha}^2}{\alpha!}-\frac{1}{n^2}\leq& \sum_{d=1}^D \sum_{m_1,m_2=2}^{d+1}n^{m_1+m_2-2}d^{2d}\pa{\frac{\rho}{1-\rho}d\pa{1+d/2}}^{d}\pa{\frac{\sqrt{2}}{n-1-d}}^{m_1+m_2}\\
    \leq& \frac{1}{n^2}\sum_{d=1}^D \sum_{m_1,m_2=2}^{d+1}\pa{\frac{D^3\rho}{1-\rho}\pa{1+\frac{D}{2}}}^d\pa{\frac{1}{1-\frac{D+1}{n}}}^{2d+2}\\
    \leq& \frac{2}{\pa{n-1-D}^2}\sum_{d=1}^D d^2\pa{\frac{D^3\rho}{1-\rho}\pa{1+D/2}\frac{2}{\pa{1-\frac{D+1}{n}}^2}}^d\\
    \leq& \frac{2}{\pa{n-1-D}^2}\sum_{d=1}^D d^2\zeta^d\\
    \leq& \frac{2}{\pa{n-1-D}^2}\zeta\frac{1+\zeta}{1-\zeta},
\end{align*}

whenever $\zeta:=\frac{D^3\rho}{1-\rho}\pa{1+D/2}\frac{2}{\pa{1-\frac{D+1}{n}}^2}<1$.

\subsection{General case $p$}

In order to lower-bound $\MMSE_{\leq D}$ in the general case with $p$ Gaussian Correlated Graphs with correlation $\rho$, we shall reduce to the problem of aligning two correlateed Gaussian graphs with correlation $\sqrt{\rho}$. Given any polynomial $f\in \mathbb{R}_{D}[G_1, \ldots, G_p]$, we construct a random polynomial $g\in \mathbb{R}_D\cro{G'_1, G'_2}$, with $G_1', G_2'$ having correlation $\sqrt{\rho}$ and being permuted with the same permutation $\pi^\star_1, \pi^\star_2$, and such that $E\cro{\pa{f-x}^2}=\E\cro{\pa{g-x}^2}$. Then, using the result for the case $p=2$, we can prove a lower bound on $E\cro{\pa{f-x}^2}$. This being valid for all polynomial of degree at most $D$, we deduce from this a lower bound on $\MMSE_{\leq D}$.

Let $f\in \mathbb{R}_D[G_1, \ldots, G_p]$ any polynomial and let us lower-bound $\E\cro{\pa{f(G)-x}^2}$. We recall that we denote $\pi^\star_1, \ldots, \pi^\star_p$ the hidden partitions and that $x=\1\ac{\pi^\star_1(1)=\pi^\star_2(1)}$. Let us generate $G_1', G_2'$ two Gaussian graphs with, for all $e\in E$, $\E\cro{G_{1e}' G_{2e'}'}=\sqrt{\rho}\1\ac{\Pi^\star_1(e)=\Pi^\star_2(e')}$. Then, sample $K_1, \ldots, K_p$ as follows; \begin{enumerate}
    \item Take $K_1=G'_1$;
    \item Independently from all the rest, sample $Z_2, \ldots, Z_p$ independent gaussian graphs;
    \item Take $K_2=\sqrt{\rho}G'_2+\sqrt{1-\rho}Z_2$;
    \item For $j\in [3,p]$, take $K_{je}= \sqrt{\rho}G'_{2\Pi^\star_j(e)}+\sqrt{1-\rho} Z_{je}$.
\end{enumerate} 
We remark that $\pa{x,K_1, \ldots, K_p}$ has the same joint law as $\pa{x,G_1, \ldots,G_p}$. Let $g(G_1', G_2')=f(K_1,\ldots, K_p)$ a random polynomial of degree at most $D$. Then, plugging the result from the case $p=2$ with a correlation $\sqrt{\rho}$ instead of $\rho$, we get that, whenever $$\zeta:=\frac{D^3\sqrt{\rho}}{1-\sqrt{\rho}}\pa{1+D/2}\frac{2}{\pa{1-\frac{D+1}{n}}^2}<1,$$ then, 
$$\E\cro{\pa{g\pa{G_1', G_2'}-x}^2}\geq \frac{1}{n}-\frac{1}{n^2}-\frac{2}{\pa{n-1-D}^2}\zeta\frac{1+\zeta}{1-\zeta}\, .$$
Thus, under the same condition on $\zeta$, $$\E\cro{\pa{f\pa{G_1,\ldots, G_p}-x}^2}=\E\cro{\pa{f\pa{K_1,\ldots, K_p}-x}^2}=\E\cro{\pa{g\pa{G_1', G_2'}-x}^2}\geq \frac{1}{n}-\frac{1}{n^2}-\frac{2}{\pa{n-1-D}^2}\zeta\frac{1+\zeta}{1-\zeta}\, .$$
This being valid for all polynomial $f$ of degree at most $D$, we conclude the proof of the theorem.

\subsection{Proof of Lemma \ref{lem:controlcumulant}}\label{prf:controlcumulant}

Let us fix $\alpha=\pa{\alpha_1, \alpha_2}$ and let us upper bound $|\kappa_{x, \alpha}|$. The first ingredient is to apply Theorem 2.5 of \cite{even2025a} which we adapt here with the specific notation of our problem. We recall that, for any edge $e=\ac{u,u'}\in E$, we denote $\Pi^\star_1(e)=\ac{\pi^\star_1(u), \pi^\star_1(u')}$ (resp. $\Pi^\star_2(e)=\ac{\pi^\star_2(u), \pi^\star_2(u')}$).

\begin{prop}\label{prop:EGV25}\cite{even2025a}
    If $|\alpha_1|\neq |\alpha_2|$, then $\kappa_{x, \alpha}=0$. If $|\alpha_1|=|\alpha_2|$, suppose that for all bijection $\psi: \alpha_1\to \alpha_2$ (where we see $\alpha_1$ and $\alpha_2$ as multisets), we have $$|\cumul\pa{x, \pa{\Pi^\star_1(e)=\Pi^{\star}_2(\psi(e))}_{e\in \alpha_1}}|\leq C_\alpha,$$
    then, 
    $$|\kappa_{x, \alpha}|\leq \pa{\frac{|\alpha|}{2}}^{|\alpha|/2}C_{\alpha}\, .$$
\end{prop}

In the light of Proposition \ref{prop:EGV25}, we shall suppose that $|\alpha_1|=|\alpha_2|$ and we shall fix $\psi$ a bijection from the multiset $\alpha_1$ to the multiset $\alpha_2$. Let us control 
$$C_{\psi, \alpha}=\cumul\pa{x, \pa{\1\ac{\Pi^{\star}_1(e)=\Pi^{\star }_2(\psi(e))}}_{e\in \alpha_1}}\, .$$ 
We refer to Section \ref{prf:reducedcumulant} for a proof of the next lemma.

\begin{lemma}\label{lem:reducedcumulant}
    For all $\alpha=\pa{\alpha_1,\alpha_2}$ and all bijection $\psi:\alpha_1\to \alpha_2$, we have $$|C_{\psi, \alpha}|\leq (1+|\alpha|/2)^{|\alpha|/2}\pa{\frac{\sqrt{2}}{\sqrt{n-1-|\alpha|}}}^{|\ac{1}\cup \supp(\alpha_1)|+|\ac{1}\cup \supp(\alpha_2)|}\, .$$  
\end{lemma}
Combing Lemma \ref{lem:reducedcumulant} and Proposition \ref{prop:EGV25}
, we arrive at $$|\kappa_{x,\alpha}|\leq \pa{|\alpha|(1+|\alpha|/2)}^{|\alpha|/2}\pa{\frac{\sqrt{2}}{\sqrt{n-1-|\alpha|}}}^{|\ac{1}\cup \supp(\alpha_1)|+|\ac{1}\cup \supp(\alpha_2)|},$$
which concludes the proof of the lemma.

\subsection{Proof of Lemma \ref{lem:reducedcumulant}}\label{prf:reducedcumulant}
By symmetry of $\alpha_1$ and $\alpha_2$, it is sufficient to prove that, for all $\alpha=\pa{\alpha_1,\alpha_2}$ and all bijection $\psi:\alpha_1\to \alpha_2$, we have $$|C_{\psi, \alpha}|\leq (1+|\alpha_1|)^{|\alpha_1|}\pa{\frac{\sqrt{2}}{n-1-|\alpha|}}^{|\ac{1}\cup \supp(\alpha_1)|}\, .$$
For $\beta\subseteq \alpha_1$, we write $$C_{\psi,(\beta, \psi(\beta))}:=\cumul\pa{x, \pa{\1\ac{\Pi^{\star}_1(e)=\Pi^{\star }_2(\psi(e))}}_{e\in \beta}}\, .$$
We shall proceed by induction over $\alpha_1$ and prove that, for all $\beta\subseteq \alpha_1$, we have $$|C_{\psi,(\beta, \psi(\beta))}|\leq \pa{1+|\beta|}^{|\beta|}\pa{\frac{\sqrt{2}}{n-1-|\alpha|}}^{|\ac{1}\cup \supp(\beta)|}\, .$$ 

\paragraph{Initialization:} $C_{x, 0, 0}=\cumul(x)=\E\cro{x}=\frac{1}{n}$. 

\paragraph{Induction step:} Let $\varnothing\subsetneq \beta\subseteq \alpha_1$ and suppose that, for all $\gamma\subsetneq \beta$, we have $$|C_{\psi,(\gamma, \psi(\gamma))}|\leq \pa{1+|\gamma|}^{|\gamma|}\pa{\frac{\sqrt{2}}{n-1-2|\gamma|}}^{|\ac{1}\cup \supp(\gamma)|},$$
and let us prove that this inequality holds for $\beta$. Using Equation \eqref{eq:recursioncumulantgeneral}, we express the cumulant $C_{\psi, (\beta, \psi(\beta))}$ as a function of some mixed moments and of the cumulants $\pa{C_{\psi, (\gamma,\psi(\gamma))}}_{\gamma\subsetneq \beta}$:
\begin{align}\nonumber
    C_{\psi, (\beta, \psi(\beta))}=&\E\cro{x\prod_{e\in \beta}\1\ac{\Pi^{\star}_1(e)=\Pi^{\star}_2(\psi(e))}}\\
    &-\sum_{\gamma\subsetneq \beta}C_{\psi, (\gamma, \psi(\gamma))}\E\cro{\prod_{e\in \beta\setminus \gamma}\1\ac{\Pi^{\star}_1(e)=\Pi^{\star}_2(\psi(e))}}\, .\label{eq:recursioncumulant} 
\end{align}
In the light of Equation \eqref{eq:recursioncumulant}, it is sufficient to upper-bound the mixed moments of the random variables $x, \pa{\1\ac{\Pi^{\star}_1(e)=\Pi^{\star}_2(\psi(e))}}_{e\in \beta}$.
\begin{lemma}\label{lem:controlmoments}
For all $\gamma\subsetneq \beta$;
\begin{enumerate}
    \item $|\E\cro{x\prod_{e\in \beta}\1\ac{\Pi^{\star}_1(e)=\Pi^{\star}_2(\psi(e))}}|\leq \pa{\frac{\sqrt{2}}{n-|\ac{1}\cup \supp(\beta)|}}^{|\ac{1}\cup \supp(\beta)|}$
    \item $|\E\cro{\prod_{e\in \beta\setminus \gamma}\1\ac{\Pi^{\star }_1(e)=\Pi^{\star}_2(\psi(e))}}|\leq \pa{\frac{\sqrt{2}}{n-| \supp(\beta\setminus \gamma)|}}^{| \supp(\beta\setminus \gamma)|}$. 
\end{enumerate}
\end{lemma}
Plugging Lemma \ref{lem:controlmoments} in Equality \eqref{eq:recursioncumulant}, together with the induction hypothesis, leads to

\begin{align*}
    |C_{\psi, (\beta, \psi(\beta)}|\leq& 
    \pa{\frac{\sqrt{2}}{n-|\ac{1}\cup \supp(\beta)|}}^{|\ac{1}\cup \supp(\beta)|}+\\
    &+\sum_{\gamma\subsetneq \beta} \pa{1+|\gamma|}^{|\gamma|}\pa{\frac{\sqrt{2}}{n-1-2|\gamma|}}^{|\ac{1}\cup \supp(\gamma)|}\pa{\frac{\sqrt{2}}{n-|\supp(\beta\setminus \gamma)|}}^{|\supp(\beta\setminus \gamma)|}\\
    \leq& \pa{\frac{\sqrt{2}}{n-1-2|\beta|}}^{|\ac{1}\cup \supp(\beta)|}\pa{1+\sum_{\gamma\subsetneq \beta}\pa{1+|\gamma|}^{|\gamma|}}\\
    =&\pa{\frac{\sqrt{2}}{n-1-2|\beta|}}^{|\ac{1}\cup \supp(\beta)|}\pa{1+\sum_{l=0}^{|\beta|-1}\binom{|\beta|}{\ell}\pa{1+l}^{\ell}}\\
    =&\pa{\frac{\sqrt{2}}{n-1-2|\beta|}}^{|\ac{1}\cup \supp(\beta)|}\pa{1+\sum_{l=0}^{|\beta|-1}\binom{|\beta|}{\ell}|\beta|^l}\\
    \leq&\pa{\frac{\sqrt{2}}{n-1-2|\beta|}}^{|\ac{1}\cup \supp(\beta)|}\pa{\sum_{l=0}^{|\beta|}\binom{|\beta|}{\ell}|\beta|^l}\\
    =&\pa{1+|\beta|}^{|\beta|}\pa{\frac{\sqrt{2}}{n-1-2|\beta|}}^{|\ac{1}\cup \supp(\beta)|},
\end{align*}
which concludes the induction and the proof of the lemma. 

\subsection{Proof of Lemma \ref{lem:controlmoments}}\label{prf:controlmoments}
Let us prove the first point of the lemma. We recall $\supp(\beta)=\ac{u\in [n], \pa{\beta}_{u:}\neq 0\quad or\quad \pa{\beta}_{:u}\neq 0}$ is the set of nodes of positive degree when seeing $\beta$ as a multigraph with vertex set $[n]$. We seek to control $\left|\E\cro{x\prod_{e\in \beta}\1\ac{\Pi^{\star}_1(e)=\Pi^{\star}_2(\psi(e))}}\right|$. 

Let us prove by induction that, for all $\beta\subseteq \alpha_1$, $$\left|\E\cro{x\prod_{e\in \beta}\1\ac{\Pi^{\star}_1(e)=\Pi^{\star}_2(\psi(e))}}\right|\leq \pa{\frac{\sqrt{2}}{n-|\ac{1}\cup \supp(\beta)|}}^{|\ac{1}\cup \supp(\beta)|}\, .$$

\paragraph{Initialization:} $\E\cro{x}=\frac{1}{n}$.

\paragraph{Induction  Step:} Let $\beta\subseteq \alpha_1$. Let us suppose that the result holds for all $\gamma\subsetneq \beta$ and let us prove that it still holds for $\beta$. Let us fix $e_0\in \beta$ and let us consider $\gamma_0=\beta\setminus e_0$. Then, 
\begin{align}\nonumber
\E\cro{x\prod_{e\in \beta}\1\ac{\Pi^{\star}_1(e)=\Pi^\star_2(\psi(e))}}=&\E\cro{x\prod_{e\in \gamma_0}\1\ac{\Pi^{\star}_1(e)=\Pi^\star_2(\psi(e))}}\\
&\times\P\cro{\Pi^{\star}_1(e_0)=\Pi^\star_2(\psi(e_0)) \big|\quad  x\prod_{e\in \gamma_0}\1\ac{\Pi^{\star}_1(e)=\Pi^\star_2(\psi(e))}=1}\, .\label{eq:inductionmoment}
\end{align}

Let us work conditionally on the event $\mathcal{X}:=\ac{x\prod_{e\in \gamma_0}\1\ac{\Pi^{\star}_1(e)=\Pi^\star_2(\psi(e))}=1}$ (supposing it is of non-zero probability) and let us upper-bound $\P\cro{\Pi^{\star}_1(e_0)=\Pi^\star_2(\psi(e_0)) \big|  \mathcal{X}}$. We distinguish three cases, according to the number of additional nodes brought by the edge $e_0:=\ac{u_0, u'_0}$;\begin{itemize}
    \item If $\ac{u_0, u'_0}\subset \ac{1}\cup \supp(\gamma_0)$, then we have the trivial upper-bound $$\P\cro{\Pi^{\star}_1(e_0)=\Pi^\star_2(\psi(e_0)) \big| \mathcal{X}}\leq 1,$$
    \item If $\left|\ac{u_0, u'_0}\setminus \ac{\ac{1}\cup \supp(\gamma_0)}\right|=1$, suppose by symmetry $u_0\notin \ac{\ac{1}\cup \supp(\gamma_0)}$. Conditionally on $\mathcal{X}$, we know that, almost surely, $\pi_1^\star(u_0)\notin \ac{1}\cup \pa{\cup_{e\in \gamma_0}\psi(e)}$ and $\pi_1^\star(u'_0)\in \ac{1}\cup \pa{\cup_{e\in \gamma_0}\psi(e)}$. Hence, for having a non-zero probability $\P\cro{\Pi^{\star}_1(e_0)=\Pi^\star_2(\psi(e_0)) \big|\mathcal{X}}$, we need to have $\psi(e_0)=(v_0, v'_0)$ with $v_0\notin \ac{1}\cup \pa{\cup_{e\in \gamma_0}\psi(e)}$ and $v'_0\in \ac{1}\cup \pa{\cup_{e\in \gamma_0}\psi(e)}$. In that case,, 
    \begin{align*}
    \P\cro{\Pi^{\star}_1(e_0)=\Pi^\star_2(\psi(e_0))\big|\mathcal{X}}=&\P\cro{\pi^{\star}_1(u_0)=v_0 |\mathcal{X}}\\
    =&\frac{1}{n-|\ac{1}\cup \supp(\gamma_0)|}\, .
    \end{align*}

    \item If $\left|\ac{u_0, u'_0}\setminus  \ac{1}\cup \supp(\gamma_0)\right|=2$, then, for having $\P\cro{\Pi^{\star}_1(e_0)=\Pi^\star_2(\psi(e_0)) \big|\mathcal{X}}\neq 0$, we need to have $\psi(e_0)=\ac{v_0, v'_0}$ with both $v_0, v'_0\notin \ac{1}\cup \pa{\cup_{e\in \gamma_0}\psi(e)}$. In that case, conditionally on $\mathcal{X}$, 
    \begin{align*}
    \P\cro{\Pi^{\star}_1(e_0)=\Pi^\star_2(\psi(e_0))\big|\mathcal{X}}=&\P\cro{\pi^{\star}_1(u_0)=v_0 \quad and\quad \pi^{\star}_1(u'_0)=v'_0\big|\mathcal{X}}\\
    &+\P\cro{\pi^{\star}_1(u_0)=v'_0 \quad and\quad \pi^{\star}_1(u'_0)=v_0\big|\mathcal{X}}\\
    =&\frac{2}{\pa{n-|\ac{1}\cup \supp(\gamma_0)|}\pa{n-1-|\ac{1}\cup \supp(\gamma_0)|}}\, .
    \end{align*}
\end{itemize}

For all those three cases, we have $$\P\cro{\Pi^{\star}_1(e_0)=\Pi^\star_2(\psi(e_0))\big|\mathcal{X}}\leq \pa{\frac{\sqrt{2}}{n-|\ac{1}\cup \supp(\beta)|}}^{|\supp(e_0)\setminus \pa{\ac{1}\cup \supp(\gamma_0)}|}\, .$$
Plugging this in Equation \eqref{eq:inductionmoment} together with the induction hypothesis, we end up with $$\E\cro{x\prod_{e\in \beta}\1\ac{\Pi^{\star}_1(e)=\Pi^\star_2(\psi(e))}}\leq \pa{\frac{\sqrt{2}}{n-|\ac{1}\cup \supp(\beta)|}}^{|\ac{1}\cup \supp(\beta)|},$$ which concludes the induction and the proof of the first point of the lemma. For the second point of the lemma, one could carry on the exact same proof, changing only the initialization step with $\E\cro{1}=1$ for $\beta$ empty.

\section{Proof of Proposition \ref{prop:reductionpermutationpartnership}}\label{prf:reductionpermutationpartnership}

The proof of this proposition is adapted from the proof of Proposition 2.1 of \cite{even2025b}. Given a $p$-tuple of permutation $\pi=\pi_1,\ldots, \pi_p$, we shall write $M^\pi$ the matrix defined by $M^\pi_{(j,u), (j',u')}=\1\ac{\pi_j(u)=\pi_{j'}(u')}$. We write $M^\star$ for $\pi^\star$. Then, 

\begin{equation*}
    p(p-1)n^2 \MMSE_{poly}=\E\cro{\|M^{\star}\|^2_F}-\corr^2_{poly}=np^2-\corr^2_{poly},
\end{equation*}
where we define 

\begin{equation*}\corr^2_{poly}=\underset{\E\cro{\|\widehat{M}\|_F^2}=1}{\sup_{\widehat{M} \quad poly-time}}\E\cro{\<M^\star, \widehat{M}\>_F}^2\, .
\end{equation*}

Supposing that $\MMSE_{poly}=\frac{1}{n}(1-\eps)$, with $0\leq \eps \leq 1$, we deduce that $\corr_{poly}^2=np+p(p-1)n\eps$. In turn, we get, for all $\widehat{\pi}$ polynomial time estimator, 
\begin{align*}
    \E\cro{\|M^{\widehat{\pi}}-M^{\star}\|^2_F}=& \E\cro{\|M^{\widehat{\pi}}\|_F^2}+\E\cro{\|M^\star\|_F^2}-2\E\cro{\<M^{\star},M^{\widehat{\pi}}\>_F}\\
    \geq& 2np^2-2\sqrt{np^2}corr_{poly}\\
    \geq& 2np^2-2\sqrt{np^2}\sqrt{np+p(p-1)n\eps}\\
    \geq& 2np^2(1-\sqrt{\eps})\, .
\end{align*}
Using Lemma H.2 from \cite{even2025b}, we know that, for all estimator $\widehat \pi$, $$\cro{1-\err(\widehat{\pi}, \pi^\star)}^2\leq 1-\frac{\|M^{\widehat\pi}-M^\star\|_F^2}{2np^2},$$
which, in turn, implies 

$$\E\cro{\pa{1-\err(\widehat{\pi}, \pi^\star)}^2}\leq \sqrt{\eps}\, .$$

This concludes the proof of the proposition.

\end{document}